\documentclass[twoside,11pt]{article}

\usepackage{graphicx, subfigure}
\usepackage[top=1in,bottom=1in,left=1in,right=1in]{geometry}
\usepackage[sort&compress]{natbib} 
\usepackage{amsfonts, amsmath, amssymb, amsthm}

\usepackage{color}
\definecolor{darkred}{RGB}{100,0,0}
\definecolor{darkgreen}{RGB}{0,100,0}
\definecolor{darkblue}{RGB}{0,0,150}

\usepackage{hyperref}
\hypersetup{colorlinks=true, linkcolor=darkred, citecolor=darkgreen, urlcolor=darkblue}
\usepackage{url}
\urldef\eacurl\url{http://www.math.ucsd.edu/~eariasca}

\newtheorem{thm}{Theorem}

\newtheorem{lem}{Lemma}

\def\beq{\begin{equation}}
\def\eeq{\end{equation}}
\def\beqn{\begin{eqnarray*}}
\def\eeqn{\end{eqnarray*}}
\def\bitem{\begin{itemize}}
\def\eitem{\end{itemize}}
\def\benum{\begin{enumerate}}
\def\eenum{\end{enumerate}}
\def\bmult{\begin{multline*}}
\def\emult{\end{multline*}}
\def\bcenter{\begin{center}}
\def\ecenter{\end{center}}

\newcommand{\thmref}[1]{Theorem~\ref{thm:#1}}

\newcommand{\lemref}[1]{Lemma~\ref{lem:#1}}
\newcommand{\secref}[1]{Section~\ref{sec:#1}}

\DeclareMathOperator*{\argmin}{arg\, min}
\DeclareMathOperator{\diam}{diam}
\DeclareMathOperator{\dist}{dist}

\def\cE{\mathcal{E}}
\def\cF{\mathcal{F}}
\def\cG{\mathcal{G}}

\def\cN{\mathcal{N}}

\def\cS{\mathcal{S}}

\def\cY{\mathcal{Y}}

\def\bbR{\mathbb{R}}

\def\bbZ{\mathbb{Z}}

\newcommand{\E}{\operatorname{\mathbb{E}}}
\renewcommand{\P}{\operatorname{\mathbb{P}}}

\def\eps{\varepsilon}

\DeclareMathOperator{\vol}{vol}

\newcommand{\1}{{\rm 1}\kern-0.24em{\rm I}}
\def\D{D}
\def\M{M}
\def\S{S}
\def\length{\rm length}

\begin{document}
\title{\bf On the convergence of maximum variance unfolding}
\author{
Ery Arias-Castro\footnote{Department of Mathematics, University of California, San Diego, USA}%
\quad
and 
\
Bruno Pelletier\footnote{D\'epartement de Math\'ematiques, IRMAR -- UMR CNRS 6625, Universit\'e Rennes II, France}
}

\date{\today}

\maketitle

\noindent {\bf Abstract.}
Maximum Variance Unfolding is one of the main methods for (nonlinear) dimensionality reduction.  We study its large sample limit, providing specific rates of convergence under standard assumptions.  We find that it is consistent when the underlying submanifold is isometric to a convex subset, and we provide some simple examples where it fails to be consistent.


\medskip
\vspace{0.2cm}

\section{Introduction} \label{sec:intro}
One of the basic tasks in unsupervised learning, aka multivariate statistics, is that of dimensionality reduction.  While the celebrated Principal Components Analysis (PCA) and Multidimensional Scaling (MDS) assume that the data lie near an affine subspace, modern approaches postulate that the data are in the vicinity of a submanifold.  Many such algorithms have been proposed in the past decade, for example, ISOMAP \citep{Tenenbaum00ISOmap}, Local Linear Embedding (LLE) \citep{Roweis00LLE}, Laplacian Eigenmaps \citep{Belkin03}, Manifold Charting \citep{brand2003charting}, Diffusion Maps \citep{coifman2006diffusion}, Hessian Eigenmaps (HLLE) \citep{Donoho03Hessian}, Local Tangent Space Alignment (LTSA) \citep{Zhang04LTSA}, Maximum Variance Unfolding \citep{mvu04}, and many others, some reviewed in \citep{van2008dimensionality,saul2006spectral}.

Although some variants exist, the basic setting is that of a connected domain $\D \subset \bbR^d$ isometrically embedded in Euclidean space as a submanifold $\M \subset \bbR^p$, with $p > d$.  We are provided with data points $x_1, \dots, x_n \in \bbR^p$ sampled from (or near) $\M$ and our goal is to output $y_1, \dots, y_n \in \bbR^d$ that can be isometrically mapped to (or close to) $x_1, \dots, x_n$. 

A number of consistency results exist in the literature.  
For example, \cite{bernstein2000graph} show that, with proper tuning, geodesic distances may be approximated by neighborhood graph distances when the submanifold $\M$ is geodesically convex, implying that ISOMAP asymptotically recovers the isometry when $\D$ is convex.  When $\D$ is not convex, it fails in general \citep{zha2003isometric}.  
To justify HLLE, \cite{Donoho03Hessian} show that the null space of the (continuous) Hessian operator yields an isometric embedding.  See also \citep{DiscreteHessian} for related results in a discrete setting.
\cite{smith2008convergence} prove that LTSA is able to recover the isometry, but only up to an affine transformation.
We also mention other results in the literature which show that, as the sample size increases, the output the algorithm converges to is an explicit continuous embedding. For instance, a number of papers analyze how well the discrete graph Laplacian based on a sample approximates the continuous Laplace-Beltrami operator on a submanifold \citep{belkin2005towards,vonLuxburg08,singer2006graph,hein2005graphs,MR2387773,coifman2006diffusion}, which is intimately related to the Laplacian Eigenmaps.  However, such convergence results do not guaranty that the algorithm is successful at recovering the isometry when one exists.  In fact, as discussed in detail by \cite{MR2438829} and \cite{Perrault-Joncas}, many of them fail in very simple settings.  
%In particular, \cite{MR2438829} show that the normalization in LLE, LE, DFM, HLLE and LTSA leads to deformations, making these algorithms unable to recover an isometry.

In this paper, we analyze Maximum Variance Unfolding (MVU) in the large-sample limit.  We are only aware of a very recent work of \cite{pap} that establishes that, under the assumption that $\D$ is convex, MVU recovers a distance matrix that approximates the geodesic distance matrix of the data.  Our contribution is the following.  
In \secref{converge}, we prove a convergence result, showing that the optimization problem that MVU solves converges (both in solution space and value) to a continuous version defined on the whole submanifold.  The basic assumption here is that the submanifold $\M$ is compact. 
In \secref{quant}, we derive quantitative convergence rates, with mild additional regularity assumptions.
In \secref{solution}, we consider the solutions to the continuum limit.
%We show that, when $\M \subset \bbR^p$ is isometric to a domain $\D \subset \bbR^d$ with $p > d$, the solutions flatten $\M$, in that (continuous) MVU embeds $\M$ into a $d$-dimensional Euclidean space.  \eac{I hope we can prove that.}
When $\D$ is convex, we prove that MVU recovers an isometry.  We also provide examples of non-convex $\D$ where MVU provably fails at recovering an isometry.% and confirm that with some numerical experiments.
We also prove that MVU is robust to noise, which \cite{MR2438829} show to be problematic for algorithms like LLE, HLLE and LTSA.
Some concluding remarks are in \secref{discussion}.  

\section{From discrete MVU to continuum MVU} \label{sec:converge}

In this section we state and prove a qualitative convergence result for MVU.  This result applies with only minimal assumptions and its proof is relatively transparent.   What we show is that the (discrete) MVU optimization problem converges to an explicit continuous optimization problem when the sample size increases.  The continuous optimization problem is amenable to scrutiny with tools from analysis and geometry, and that will enable us to better understand (in \secref{solution}) when MVU succeeds, and when it fails, at recovering an isometry to a Euclidean domain when it exists.  

Let us start by recalling the MVU algorithm \citep{weinberger2006introduction,mvu04,weinberger2005nonlinear}.  We are provided with data points $x_1, \dots, x_n \in \bbR^p$.  Let $\|\cdot\|$ denote the Euclidean norm.
Let $\mathcal{Y}_{n,r}$ be the (random) set defined by
$$\cY_{n,r} = \left\{y_1,\dots,y_n\in\mathbb{R}^p\,:\, \|y_i-y_j\| \leq \|x_i-x_j\| \text{ when } \|x_i-x_j\| \le r \right\}.$$
Choosing a neighborhood radius $r > 0$, MVU solves the following optimization problem:

\bcenter
Discrete MVU
\ecenter
\vskip -.1in
\begin{align} 
\text{Maximize} & \quad \cE(Y) := \frac1{n(n-1)} \sum_{i=1}^n\sum_{j\neq i} \|y_i-y_j\|^2, \quad \text{ over } Y = (y_1, \dots, y_n)^T \in \bbR^{n \times p}, \label{EY}\\
\text{subject to} & \quad Y \in \cY_{n,r}. \label{Yr}
\end{align}

\medskip
When the data points are sampled from a distribution $\mu$ with support $\M$, our main result in this section is to show that, when $\M$ is sufficiently regular and $r = r_n \to 0$ sufficiently slowly, the discrete optimization problem converges to the following continuous optimization problem:

\bcenter
Continuum MVU
\ecenter
\vskip -.1in
\begin{align} 
\text{Maximize} & \quad \cE(f) := \int_{\M \times \M} \|f(x)-f(x')\|^2 \mu({\rm d}x) \mu({\rm d}x'), \quad \text{ over } f:\M\to\bbR^p, \label{Ef}\\
\text{subject to} & \quad f \text{ is Lipschitz with $\|f\|_{\rm Lip} \leq 1$},
\end{align}
where $\|f\|_{\rm Lip}$ denotes the smallest Lipschitz constant of $f$.
It is important to realize that the Lipschitz condition is with respect to the intrinsic metric on $\M$ (i.e., the metric inherited from the ambient space $\bbR^p$), defined as follows: for $x, x' \in \M$, let
\beq \label{metric}
\delta_{\M}(x, x') = \inf\{T : \exists \gamma: [0,T] \to \M, \text{ 1-Lipschitz, with } \gamma(0) = x \text{ and } \gamma(T) = x'\}.  
\eeq
When $\M$ is compact, the infimum is attained.  In that case, $\delta_{\M}(x, x')$ is the length of the shortest continuous path on $\M$ starting at $x$ and ending at $x'$, and $(M,\delta_\M)$ is a complete metric space, also called a {\it length space} in the context of metric geometry \citep{MR1835418}.
Then $f:\M\to\bbR^p$ is Lipschitz with $\|f\|_{\rm Lip} \leq L$ if
\beq \label{Lip}
\|f(x) - f(x')\| \le L \, \delta_\M(x,x'), \ \forall x, x' \in \M.
\eeq
For any $L>0$, denote by $\mathcal{F}_L$ the class of Lipschitz functions $f:\M\to\bbR^p$ satisfying \eqref{Lip}.  

One of the central condition is that $\M$ is sufficiently regular that the intrinsic metric on $\M$ is locally close to the ambient Euclidean metric. \\[-.1in] 

\noindent {\bf Regularity assumption.}  There is a non-decreasing function $c: [0, \infty) \to [0, \infty)$ such that $c(r) \to 0$ when $r \to 0$, such that, for all $x, x' \in \M$,
\beq \label{reg}
\delta_{\M}(x, x') \le \big(1 + c(\|x - x'\|)\big) \|x - x'\|.
\eeq

This assumption is also central to ISOMAP.  \cite{bernstein2000graph} prove that it holds when $\M$ is a compact, smooth and geodesically convex submanifold (e.g., without boundary).  In \lemref{geodist}, we extend this to compact, smooth submanifolds with smooth boundary, and to tubular neighborhoods of such sets.  The latter allows us to study noisy settings.

Note that we always have
\beq \label{reg-ub}
\|x - x'\| \le \delta_{\M}(x, x').
\eeq

Let $\cS_1$ denote the set of functions that are solutions of Continuum MVU.  We state the following  qualitative result that makes minimal assumptions.
\begin{thm} \label{thm:conv}
Let $\mu$ be a (Borel) probability distribution with support $\M \subset \bbR^p$, which is connected, compact and satisfying \eqref{reg}, and assume that $x_1, \dots, x_n$ are sampled independently from $\mu$.  Then, for $r_n \to 0$ sufficiently slowly, we have
\beq \label{E-conv}
\sup\{\cE(Y) : Y \in \cY_{n, r_n}\} \to \sup\{\cE(f) : f \in \cF_1\},
\eeq
and for any solution $\hat{Y}_n = (\hat{y}_1, \dots, \hat{y}_n)$ of Discrete MVU,
\beq \label{S-conv}
\inf_{f\in\mathcal{S}_1} \max_{1\leq i \leq n}\|\hat{y}_{i} - f(x_i)\| \to 0,
\eeq
almost surely as $n \to \infty$.
%with probability tending to one as $n \to \infty$.
\end{thm}
%Note that by choosing $r_n \to 0$ as slowly as needed, we obtain an almost sure convergence.

Thus Discrete MVU converges to Continuum MVU in the large sample limit, if $\M$ satisfies the crucial regularity condition \eqref{reg} and other mild assumptions.  
In \secref{quant}, we provide explicit quantitative bounds for the convergence results \eqref{E-conv} and \eqref{S-conv} at the very end, under some additional (though natural) assumptions.   
In \secref{solution}, we focus entirely on Continuum MVU, with the goal of better understanding the functions that are solutions to that optimization problem.  Because of \eqref{S-conv}, we know that the output of Discrete MVU converges in a strong sense to one of these functions.

The rest of the section is dedicated to proving \thmref{conv}.  We divide the proof into several parts which we discuss at length, and then assemble to prove the theorem.

\subsection{Coverings and graph neighborhoods} \label{sec:packings}

For $r > 0$, let $G_r$ denote the undirected graph with nodes $x_1, \dots, x_n$ and an edge between $x_i$ and $x_j$ if $\|x_i - x_j\| \le r$.  This is the $r$-neighborhood graph based on the data.  It is essential that $G_{r_n}$ be connected, for otherwise $\sup\{\cE(Y) : Y \in \cY_{n, r_n}\} = \infty$, while $\sup\{\cE(f) : f \in \cF_1\}$ is finite.  The latter comes from the fact that, for any $f \in \cF_1$, 
\[
\cE(f) \le \int_{\M \times \M} \delta_\M(x, x')^2 \mu({\rm d}x)\mu({\rm d}x') \le \diam(\M)^2,
\]
where we used \eqref{Lip} in the first inequality, and $\diam(\M)$ is the intrinsic diameter of $\M$, i.e.,
\beq \label{diam}
\diam(\M) := \sup_{x, x' \in \M} \delta_\M(x, x').
\eeq
%To prove that $\diam(\M) < \infty$, let $r$ be such that $\delta_\M(x,x') \le 2 \|x - x'\|$ for all $x,x' \in \M$ such that $\|x - x'\| \le r$.  Such an $r$ exists by \eqref{reg}.  Since $\M$ is compact, we can cover it with a finite number of balls of radius $r/2$, with centers $z_1, \dots, z_N \in \M$.  For any $x, x' \in \M$, there is a chain $z_{j_1}, \dots, z_{j_k}$ such that $\|x - z_{j_1}\| \le r/2$, $\|z_{j_s} - z_{j_{s+1}}\| \le r/2$ and $\|x' - z_{j_k}\| \le r/2$, for otherwise it would imply that $\M$ is disconnected.  We then apply the triangle inequality, to get
%\beqn
%\delta_\M(x,x') 
%& \le &\delta_\M(x, z_{j_1}) + \sum_{s=1}^{k-1} \delta_\M(z_{j_s}, z_{j_{s+1}}) + \delta_\M(x', z_{j_k}) \\
%& \le & \|x - z_{j_1}\| + \sum_{s=1}^{k-1} \|z_{j_s} - z_{j_{s+1}}\| + \|x' - z_{j_k}\| \\
%& \le & (k+1) r/2.
%\eeqn
%Hence, $\diam(\M) \le (N+1) r/2 < \infty$.
%
Recall that the only assumptions on $M$ made in Theorem~\ref{thm:conv} are that $M$ is compact, connected, and satisfies \eqref{reg}, and this implies that $\diam(\M) < \infty$.
Indeed, as a compact subset of $\mathbb{R}^p$, $M$ is bounded, hence $\sup_{x,x'\in M} \|x-x'\| < \infty$.
Reporting this in \eqref{reg} immediately implies that $\diam(\M) < \infty$.

That said, we ask more of $(r_n)$ than simply having $G_{r_n}$ connected.  For $\eta > 0$, define
\beq \label{omega}
\Omega(\eta) = \{\forall x \in \M, \exists i = 1, \dots, n : \|x - x_i\| \le \eta\},
\eeq  
which is the event that $x_1, \dots, x_n$ forms an $\eta$-covering of $\M$. 

\medskip
\noindent {\bf Connectivity requirement.}  $r_n \to 0$ in such a way that 
\beq \label{omega-conv}
\sum_{n=1}^\infty \P\left(\Omega(\lambda_n r_n)^c\right) < \infty, \text{ for some sequence } \lambda_n \to 0.
\eeq
%\bigskip
  
Since $\M$ is the support of $\mu$, there is always a sequence $(r_n)$ that satisfy the Connectivity requirement.
To see this, for $\eta > 0$, let $z_1, \dots, z_{N_\eta}$ be an $\eta$-packing of $\M$ of maximal size $N_\eta$, i.e., a maximal collection of points such that $\|z_i-z_j\| > \eta$ for all $i\neq j$.
Recall that an $\eta$-packing is also an $\eta$-covering of $M$ and note that $N_\eta < \infty$ by compacity of $M$.
Let $p_\eta = \min_j \mu(B(z_j, \eta))$.  Since $\M$ is the support of $\mu$, $\mu(B(z, \eta)) > 0$ for any $z \in \M$ and any $\eta > 0$, where $B(z, \eta)$ denotes the Euclidean ball centered at $z$ and of radius $\eta > 0$.  Hence, $p_\eta > 0$ for any $\eta > 0$.  
We have
\beqn
\P\left(\Omega(2\eta)^c\right) & =  & \P\left(\text{there exists $x\in M\,:\,\forall i=1,\dots,n,\, \|x-x_i\|>2\eta$ }  \right)     \\%& 1 - \P(\Omega(2\eta)) \\
&\le& \P(\text{there is $j$ such that $B(z_j, \eta)$ is empty of data points}) \\
&\le& \sum_{j = 1}^{N_{\eta}} \P(\text{$B(z_j, \eta)$ is empty of data points}) \\ %(x_i \notin B(z_j, \eta), \forall i) \\
&\le& N_{\eta} (1 - p_\eta)^n.
\eeqn
Let $\eta_n = \inf\{\eta > 0: N_{\eta} (1 - p_\eta)^n \le 1/n^2\}$ ;  the sequence $1/n^2$ is chosen here for the simplicity of the exposition, but more general sequence can be considered, as will become apparent at the end of the paragraph.

Since $p_\eta > 0$ for all $\eta > 0$, $\eta_n \to 0$.
To see this, let $\eta^\star = {\rm diam}(M)$.
Clearly, for all $\eta \geq \eta^\star$, $p_\eta = 1$, which implies that the set of $\eta>0$ such that $N_{\eta} (1 - p_\eta)^n \le 1/n^2$ is non-empty.
In particular, for all $n\geq 1$, we have $\eta_n \leq \eta^\star$.
Now, let $\varepsilon>0$ be fixed. 
Since $p_\varepsilon>0$, there exists an integer $n_\varepsilon$ such that $N_{\varepsilon} (1 - p_\varepsilon)^n \le 1/n^2$ for all $n\geq n_\varepsilon$, so that $\eta_n \leq \varepsilon$ for all $n\geq n_\varepsilon$.
Since $\varepsilon$ is arbitrary, this proves that the sequence $(\eta_n)$ converges to 0 as $n$ tends to infinity.

With such a choice of $(\eta_n)$, we have $\sum_{n\geq 1}\P(\Omega(2\eta_n)^c) \leq \sum_{n\geq1}1/n^2 <\infty$.
%This implies that $\P(\Omega(2\eta_n)) \to 1$.
Therefore, if we take $r_n = \sqrt{\eta_n}$, it satisfies the Connectivity requirement.
In \secref{radius} we derive a quantitative bound on $r_n$ that guaranty \eqref{omega-conv} under additional assumptions.
Note that the sequence $(1/n^2)$ in the definition of $\eta_n$ can be replaced by any summable decreasing sequence.
%For example, when $\M$ is a compact $d$-dimensional submanifold and $\mu$ is the uniform distribution on $\M$, we need $r_n \ge C (\log(n)/n)^{1/d}$ for a large enough constant $C$ depending on $\M$.  See \lemref{omegan} for more details.  

The rationale behind the requirement on $(r_n)$ is the same as in \citep{bernstein2000graph}: it allows to approximate each curve on $\M$ with a path in $G_{r_n}$ of nearly the same length.  We utilize this in the following subsection.

\subsection{Interpolation}

Assuming that the sampling is dense enough that $\Omega(\eta)$ holds, we interpolate a set of vectors $Y \in \cY_{n, r}$ with a Lipschitz function $f \in \cF_{1 + O(\eta/r)}$.  Formally, we have the following.

\begin{lem} \label{lem:extension}
Assume that $\Omega(\eta)$ holds $\eta \le r/4$.  Then any vector $Y = (y_1, \dots, y_n) \in \cY_{n, r}$ is of the form $Y = (f(x_1), \dots, f(x_n))$ for some $f \in \cF_{1 + 6 \eta/r}$.
\end{lem}

We prove this result.  The first step is to show that this is at all possible in the sense that 
\beq \label{lip-discrete}
\|y_i - y_j\| \leq \big(1 + 6\eta/r \big) \delta_\M(x_i,x_j), \ \forall i,j.
\eeq
%Let $\cX_n = \{x_1, \dots, x_n\} \subset \bbR^p$ be the data
This shows that the map $g: \{x_1,\dots,x_n\} \to \bbR^p$ defined by $g(x_i) = y_i$ for all $i$, is Lipschitz (for $\delta_\M$ and the Euclidean metrics) with constant $L = 1 + 6\eta/r$.  We apply a form of Kirszbraun's Extension --- \citep[Th.~B]{MR1466337} or \citep[Th.~1.26]{MR2882877} --- to extend $g$ to the whole $\M$ into $f \in \cF_{1+6\eta/r}$. 

Therefore, let's turn to proving \eqref{lip-discrete}.  The arguments are very similar to those in \citep{bernstein2000graph}.  If $\delta_\M(x_i,x_j) \leq r$, then, by \eqref{reg-ub}, $\|x_i-x_j\| \leq r$, which implies that
$$\|y_i-y_j\| \leq \|x_i-x_j\| \leq \delta_\M(x_i,x_j).$$

\def\nn{\zeta}
Now suppose that $\delta_\M(x_i,x_j) > r$.
Let $\gamma$ be a path in $\M$ connecting $x_i$ to $x_j$ of minimal length $l = \delta_\M(x_i,x_j)$.
Split $\gamma$ into $N$ arcs of lengths $l_1=r/2$ plus one arc of length $l_{N+1}<l_1$, so that 
$$\frac{l}{l_1} - 1 \leq N \leq \frac{l}{l_1}.$$
Denote by $x_i=x_0', x_1',\dots,x_N',x_{N+1}'=x_j$ the extremities of the arcs along $\gamma$.

%For any $x$ in $\M$, let $i_x \in \argmin_i \|x - x_i\|$.  
For $k = 1, \dots, N$, let $t_k \in \argmin_t \|x_k' - x_t\|$.
On $\Omega_n(\eta)$, $\delta_\M(x_k', x_{t_k}) \leq \eta$ for all $k$, so that
$$\|x_{t_k} -x_{t_{k-1}}\| \le \delta_\M(x_{t_k},x_{t_{k-1}}) \leq \delta_\M(x'_k,x'_{k-1}) + 2\eta \leq l_1 + 2\eta \leq r/2+2(r/4) = r.$$
Hence, because $Y = (y_1 \dots, y_n) \in \cY_{n, r}$,
$$\|y_{t_k} -y_{t_{k-1}}\| \leq l_1 + 2\eta.$$
Similarly, for the last arc, recalling that $x_{t_{N+1}} = x_j$, we have $\delta_\M(x_j , x_{t_N}) = l_{N+1} +\eta < l_1 +\eta <r$, and therefore
$$\|y_{t_{N+1}} -y_{t_N}\| \leq l_{N+1} + \eta.$$
Consequently,
\begin{eqnarray*}
\|y_i-y_j\| & \leq & N(l_1 + 2\eta) + (l_{N+1}+\eta) \\
 & = & N l_1 +l_{N+1} + (2N+1)\eta \\
 & = & l + (2N+1)\eta.
\end{eqnarray*}
We have
$$(2N+1)\eta \leq \left(2\frac{l}{l_1} +1\right)\eta \leq l \frac{3\eta}{l_1} = l \frac{6\eta}{r}, $$
and so \eqref{lip-discrete} holds.

\subsection{Bounds on the energy}

We call $\cE$ the energy functional.
For a function $f: \{x_1, \dots, x_n\} \to \bbR^p$, let $Y_n(f) = (f(x_1), \dots, f(x_n))^T \in \bbR^{n \times p}$.  Assume that $\Omega(\eta)$ holds $\eta \le r/4$.  Then \lemref{extension} implies that any $Y \in \cY_{n, r}$ is equal to $Y(f)$ for some $f \in \cF_{1 + 6 \eta/r}$.  Hence,
\beq \label{E-ub}
\sup_{Y \in \cY_{n, r}} \cE(Y) \le \sup_{f \in \cF_{1+6\eta/r}} \cE(Y_n(f)).
\eeq

Recall the function $c(r)$ introduced in \eqref{reg}, and assume that $r > 0$ is small enough that $c(r) < 1$.
For $f \in \cF_{1 - c(r)}$, and for any $i,j$ such that $\|x_i - x_j\| \le r$, we have 
$$\|f(x_i) - f(x_j)\|  \leq  (1 -c(r)) \delta_\M(x_i, x_j)   \leq  (1-c(r))(1+c(\|x_i-x_j\|)) \|x_i-x_j\|.$$
Since the function $c$ is non-decreasing, $c(\|x_i-x_j\|) \leq c(r)$, and so
$$\|f(x_i) - f(x_j)\|  \leq  \left(1-c(r)^2\right) \|x_i-x_j\| \leq \|x_i-x_j\|.$$
Consequently, $Y_n(f) \in \cY_{n,r}$, implying that
\beq \label{E-lb}
\sup_{Y \in \cY_{n, r}} \cE(Y) \ge \sup_{f \in \cF_{1-c(r)}} \cE(Y_n(f)).
\eeq

As a result of \eqref{E-ub} and \eqref{E-lb}, we have
\beq \label{E-ineq}
\big| \sup_{Y \in \cY_{n, r}} \cE(Y) - \sup_{f \in \cF_1} \cE(f) \big| \le \sup_{1-c(r) \le L \le 1+6\eta/r} \big| \sup_{f \in \cF_L} \cE(Y_n(f)) - \sup_{f \in \cF_1} \cE(f) \big|.
\eeq
%Since $\cE(Y_n(f))$ is the sample version of $\cE(f)$, we apply the triangle inequality to arrive at
We have 
$$\big| \sup_{f\in\cF_L}\cE(Y_n(f)) -  \sup_{f\in\cF_L}\cE(f) \big|  \leq \sup_{f\in\cF_L}\big| \cE(Y_n(f)) -\cE(f)   \big|,$$
and applying the triangle inequality, we arrive at  
%\beq \label{E-tri}
$$\big| \sup_{f \in \cF_L} \cE(Y_n(f)) - \sup_{f \in \cF_1} \cE(f) \big| \le \sup_{f \in \cF_L} \big|\cE(Y_n(f)) -\cE(f) \big| + \big| \sup_{f \in \cF_L} \cE(f) - \sup_{f \in \cF_1} \cE(f) \big|.$$
%\eeq
Since $\cF_L = L \cF_1$ and $\cE(L f) = L^2 \cE(f)$, we have
%\beq \label{E-sup}
$$\big| \sup_{f \in \cF_L} \cE(f) - \sup_{f \in \cF_1} \cE(f) \big| 
\le |L^2-1| \sup_{f \in \cF_1} \cE(f) 
\le |L^2-1| \diam(\M)^2, $$
%\le |L-1| (L+1) \diam(\M)^2, $$
%\eeq
and
\beq \label{emp}
\sup_{f \in \cF_L} \big|\cE(Y_n(f)) -\cE(f) \big| = L^2 \sup_{f \in \cF_1} \big|\cE(Y_n(f)) -\cE(f) \big|.
\eeq
Consequently,
$$\big| \sup_{f \in \cF_L} \cE(Y_n(f)) - \sup_{f \in \cF_1} \cE(f) \big|  \leq L^2 \sup_{f \in \cF_1} \big|\cE(Y_n(f)) -\cE(f) \big| +  |L^2-1| \diam(\M)^2.$$
Reporting this inequality in \eqref{E-ineq} on the event $\Omega(\eta)$ with $\eta \leq r/4$, we have
\begin{equation}
\label{eq:ineqbase}
\big| \sup_{Y \in \cY_{n, r}} \cE(Y) - \sup_{f \in \cF_1} \cE(f) \big| \leq (1+6\eta/r)^2 \sup_{f \in \cF_1} \big|\cE(Y_n(f)) -\cE(f) \big| + \beta(r,\eta)\big(2+\beta(r,\eta)\big) \diam(\M)^2,
\end{equation}
where $\beta(r,\eta) := \max(c(r),6\eta/r)$.

Finally, we show that $\cE$ is continuous (in fact Lipschitz) on $\cF_1$ for the supnorm.  For any $f$ and $g$ in $\mathcal{F}_1$, and any $x$ and $x'$ in $\M$, we have:
\begin{eqnarray*}
\left|\|f(x)-f(x')\|^2 - \|g(x)-g(x')\|^2\right| & \leq & \|f(x)-f(x')-g(x)+g(x')\| \|f(x)-f(x')+g(x)-g(x')\| \\
 & \leq & \left[ \|f(x)-g(x)\| + \|f(x')-g(x')\|\right] \\
 & & \times\left[ \|f(x)-f(x')\| + \|g(x)-g(x')\| \right]\\
 & \leq &  4 \|f-g\|_\infty \diam(\M).
\end{eqnarray*}
The first inequality is that of Cauchy-Schwarz.
Hence,
\beq \label{E-lip}
\big| \cE(f) - \cE(g) \big| \le 4 \|f-g\|_\infty \diam(\M),
\eeq
and
\beq \label{En-lip}
\big| \cE(Y_n(f)) - \cE(Y_n(g)) \big| \le 4 \|f-g\|_\infty \diam(\M).
\eeq

\subsection{More coverings and the Law of Large Numbers}

The last step is to show that the supremum of the empirical process \eqref{emp} converges to zero.  For this, we use a packing (covering) to reduce the supremum over $\cF_1$ to a maximum over a finite set of functions.  We then apply the Law of Large Numbers to each difference in the maximization.  

Fix $x_0 \in \M$ and define 
\[
\cF_1^0 = \{f \in \cF_1: f(x_0) = 0\}.
%\cF_1^0 = \{f \in \cF_1: \exists x_f \in \M \text{ s.t. } f(x_f) = 0\}.
\]
Note that $f \in \cF_1$ if, and only if, $f - f(x_0) \in \cF_1^0$, and by the fact that $\cE(f + a) = \cE(f)$ for any function or vector $f$ and any constant $a \in \bbR^p$, we have
\[
\sup_{f \in \cF_1} \big|\cE(Y_n(f)) -\cE(f) \big| = \sup_{f \in \cF_1^0} \big|\cE(Y_n(f)) -\cE(f) \big|.
\]
The reason to use $\cF_1^0$ is that it is bounded in supnorm.  Indeed, for $f \in \cF_1^0$, we have
\[
\|f(x)\| = \|f(x) - f(x_0)\| \le \delta_\M(x,x_0) \le \diam(\M), \ \forall x \in \M.
\] 

Let $\cN_\infty(\cF_1^0, \eps)$ denote the covering number of $\cF_1^0$ for the supremum norm, i.e., the minimal number of balls that are necessary to cover $\cF_1^0$, and let $f_1,\dots,f_{N} \in \cF_1$ be an $\eps$-covering of $\cF_1^0$ of minimal size $N := \cN_\infty(\cF_1^0, \eps)$.  Since $\cF^0_1$ is equicontinuous and bounded, it is compact for the topology of the supremum norm by the Arzel\`a-Ascoli Theorem, so that $\cN_\infty(\cF_1^0, \eps) < \infty$ for any $\eps > 0$.  

Fix $f \in \cF_1^0$ and let $k$ be such that $\|f - f_k\| \le \eps$.  By \eqref{E-lip} and \eqref{En-lip}, we have
\begin{eqnarray*}
\left|\cE(Y_n(f)) - \mathcal{E}(f)\right| & \leq & \left|\cE(Y_n(f)) - \cE(Y_n(f_k))\right| + \left|\cE(Y_n(f_k)) - \mathcal{E}(f_k)\right| + \left|\mathcal{E}(f_k) - \mathcal{E}(f)\right| \\
 & \leq & 8 \diam(\M)\|f-f_k\|_\infty + \left|\cE(Y_n(f_k)) - \mathcal{E}(f_k)\right| \\
 & = & 8\diam(\M)\eps + \left|\cE(Y_n(f_k)) - \mathcal{E}(f_k)\right|. 
\end{eqnarray*}
Thus,
\beq \label{emp-bound}
\sup_{f \in \cF_1} \big|\cE(Y_n(f)) -\cE(f) \big| \le 8\diam(\M)\eps + \max\{\left|\cE(Y_n(f_k)) - \mathcal{E}(f_k)\right| : k=1,\dots,\cN_\infty(\cF_1^0, \eps)\}.
\eeq

The Law of Large Numbers (LLN) imply that, for any bounded $f$, $\cE(Y_n(f)) \to \cE(f)$, almost surely as $n \to \infty$.  Indeed,
\beqn
\cE(Y_n(f)) &=& \frac{n^2}{n(n-1)}\frac1{n^2} \sum_{i, j} \|f(x_i) - f(x_j)\|^2 \\
&=& \frac{2n}{n-1}\left[\frac1n \sum_i \|f(x_i)\|^2 - \left\|\frac1n \sum_i f(x_i)\right\|^2 \right]\\
&\to& 2\E \|f(x)\|^2 - 2\|\E f(x) \|^2 = \cE(f), \quad \text{ almost surely as } n \to \infty,
\eeqn
by the LLN applied to each term.
Therefore, when $\eps > 0$ is fixed, the second term in \eqref{emp-bound} tends to zero almost surely, and since $\eps > 0$ is arbitrary, we conclude that
\beq \label{emp-lim}
\sup_{f \in \cF_1} \big|\cE(Y_n(f)) -\cE(f) \big| \to 0, \text{in probability, as } n \to \infty.
\eeq

\subsection{Large deviations of the sample energy}

To show an almost sure convergence in \eqref{emp-lim}, we need to refine the bound on the supremum of the empirical process \eqref{emp}.  For this, we apply Hoeffding's Inequality for U-statistics \citep{hoeffding}, which is a special case of \cite[Thm. 4.1.8]{MR1666908}.

\begin{lem}[Hoeffding's Inequality for U-statistics] \label{lem:ustat}
Let $\phi:\M\times \M\to\bbR$ be a bounded measurable map, and let $\{x_i\,:\,i\geq 1\}$ be a sequence of i.i.d.~random variables with values in $\M$.
Assume that $\mathbb{E}[\phi(x_1,x_2)]=0$ and that $b:=\|\phi\|_\infty < \infty$, and let $\sigma^2=\operatorname{Var}(\phi(x_1,x_2))$.
Then, for all $t>0$, 
$$\mathbb{P}\left[ \frac{1}{n(n-1)}\sum_{1\leq i\neq j\leq n}\phi(x_i,x_j) > t \right] \leq \exp\left(-\frac{nt^2}{5\sigma^2+3bt}\right).$$
\end{lem}

Let $f \in \cF_1$.  To bound the deviations of $\cE(Y_n(f))$, we apply this result with $\phi(x,x')=\|f(x)-f(x')\|^2 - \cE(f)$.  
Then,
$$\cE(Y_n(f)) - \cE(f) = \frac{1}{n(n-1)}\sum_{i\neq j} \phi(x_i,x_j).$$
By construction, $\E[\phi(x_1,x_2)] = 0$.
Since $f$ is Lipschitz with constant 1, for any $x$ and $x'$ in $\M$, $\|f(x)-f(x')\|^2 \leq\diam(\M)^2$ and $\mathcal{E}(f) \leq \diam(\M)^2$.
Hence $\|\phi\|_\infty \leq \diam(\M)^2$, and $\operatorname{Var}(\phi(x_1,x_2)) \leq \|\phi\|^2_\infty \leq \diam(\M)^4$.
Applying \lemref{ustat} (twice), we deduce that, for any $\eps>0$,
\begin{equation}
\label{eq:concentration1}
\mathbb{P}\left(\left|\cE(Y_n(f)) - \cE(f)\right| > \eps\right) \leq 2\exp\left(-\frac{n \eps^2}{5 \diam(\M)^4 + 3 \diam(\M)^2 \eps}\right).
\end{equation}

Using \eqref{eq:concentration1} in \eqref{emp-bound}, coupled with the union bound, we get that 
\beq \label{ld}
\P\left(\sup_{f \in \cF_1} \big| \cE(Y_n(f)) - \cE(f) \big| > 9 \eps \diam(\M) \right) \le \cN_\infty(\cF_1^0,\eps) \cdot 2 \exp\left(-\frac{n \eps^2}{5 \diam(\M)^2 + 3 \eps}\right).
\eeq
Clearly, the RHS is summable for every $\eps > 0$ fixed, so the convergence in \eqref{emp-lim} happens in fact with probability one, that is,
\beq \label{emp-lim-as}
\sup_{f \in \cF_1} \big|\cE(Y_n(f)) -\cE(f) \big| \to 0, \text{ almost surely, as } n \to \infty.
\eeq

\subsection{Convergence in value: proof of \eqref{E-conv}} \label{sec:E-conv}

%Assume $r_n$ satisfies the Connectivity requirement, and that $n$ is large enough that $\lambda_n \le 1/2$.  When $\Omega(\lambda_n r_n)$ holds, \eqref{E-ineq} is valid, and by \eqref{E-tri} and \eqref{emp}, we have 
%\beqn
%\big| \sup_{Y \in \cY_{n, r}} \cE(Y) - \sup_{f \in \cF_1} \cE(f) \big| 
%&\le 
%4 \max(c(r_n), 6 \lambda_n) \diam(\M)^2 \1_{\Omega(\lambda_n r_n)} + (1 + 6 \lambda_n)^2 \diam(\M)^2 \1_{\Omega(\lambda_n r_n)^c} \\
%& + (1 + 6 \lambda_n)^2 \sup_{f \in \cF_1} \big|\cE(Y_n(f)) -\cE(f) \big|.  
%\eeqn
%Almost surely, the sum of the first two terms on the RHS tends to 0 by the fact that $c(r) \to 0$ when $r \to 0$, and \eqref{omega-conv}, while the third term tends to 0 by \eqref{emp-lim}.  Hence, \eqref{E-conv} is established. 
%
Assume $r_n$ satisfies the Connectivity requirement, and that $n$ is large enough that $\max(c(r_n),6\lambda_n) < 1$.
When $\Omega(\lambda_n r_n)$ holds, by \eqref{eq:ineqbase}, we have
$$\big| \sup_{Y \in \cY_{n, r}} \cE(Y) - \sup_{f \in \cF_1} \cE(f) \big|  \leq (1+6\lambda_n)^2 \sup_{f \in \cF_1} \big|\cE(Y_n(f)) -\cE(f) \big| + 3 \max\big(c(r_n),6\lambda_n\big) \diam(\M)^2,$$
while when $\Omega(\lambda_n r_n)$ does not hold, since the energies are bounded by $\diam(M)^2$, we have
$$\big| \sup_{Y \in \cY_{n, r}} \cE(Y) - \sup_{f \in \cF_1} \cE(f) \big|  \leq 2 \diam(\M)^2.$$
Combining these inequalities, we deduce that
\begin{eqnarray}
\big| \sup_{Y \in \cY_{n, r}} \cE(Y) - \sup_{f \in \cF_1} \cE(f) \big| 
&\le 
3 \max\big(c(r_n),6\lambda_n\big) \diam(\M)^2 \1_{\Omega(\lambda_n r_n)} +   2 \diam(\M)^2 \1_{\Omega(\lambda_n r_n)^c} \notag \\
& + (1 + 6 \lambda_n)^2 \sup_{f \in \cF_1} \big|\cE(Y_n(f)) -\cE(f) \big|. \label{E-diff-conv}  
\end{eqnarray}
Almost surely, the sum of the first two terms on the RHS tends to 0 by the fact that $c(r) \to 0$ when $r \to 0$, and \eqref{omega-conv} since $r_n$ satisfies the Connectivity requirement.
The third term tends to 0 by \eqref{emp-lim}.  Hence, \eqref{E-conv} is established. 

\subsection{Convergence in solution: proof of \eqref{S-conv}}

Assume $r_n$ satisfies the Connectivity requirement, and that $n$ is large enough that $\lambda_n \le 1/2$.
Let $\hat{Y}_n$ denote any solution of Discrete MVU.
When $\Omega(\lambda_n r_n)$ holds, there is $\hat{f}_n \in \cF_{1+6\lambda_n}$ such that $\hat{Y}_n = Y_n(\hat{f}_n)$.
Note that the existence of the interpolating function $\hat{f}_n$ holds on $\Omega(\lambda_n r_n)$ for each fixed $n$, and that this does not imply the existence of an interpolating sequence $(\hat{f}_n)_{n\geq 1}$.
That said, for each $\omega$ in the event $\liminf_n \Omega(\lambda_n r_n)$, there exists a sequence $\hat{f}_n(.;\omega)$ and an integer $n_0(\omega)$ such that $\hat{Y}_n = Y_n(\hat{f}_n)$ for all $n\geq n_0(\omega)$, i.e., the sequence is interpolating a solution of Discrete MVU for all $n$ large enough.
In addition, when $r_n$ satisfies the Connectivity requirement, then $\P(\limsup_n\Omega(\lambda_n r_n)^c) = 0$ by the Borel-Cantelli lemma.
Hence the event $\liminf_n \Omega(\lambda_n r_n)$ holds with probability one.

In fact, without loss of generality, we may assume that $\hat{f}_n \in \cF^0_{1+6\lambda_n} \subset \cF^0_4$.  Since $\cF^0_4$ is equicontinuous and bounded, it is compact for the topology of the supnorm by the Arzel\`a-Ascoli Theorem.  Hence, any subsequence of $\hat{f}_n$ admits a subsequence that converges in supnorm.  And since $\cF^0_L$ increases with $L$ and $\cF^0_1 = \cap_{L > 1} \cF^0_L$, any accumulation point of $(\hat{f}_n)$ is in $\cF^0_1$.  

In fact, if we define $\cS^0_1 = \cS_1 \cap \cF^0_1$, then all the accumulation points of $(\hat{f}_n)$ are in $\cS^0_1$.   Indeed, we have
\[
\cE(\hat{f}_n) = \cE(\hat{f}_n) - \cE(Y_n(\hat{f}_n)) + \cE(Y_n(\hat{f}_n)),
\]
with
\[
\left| \cE(\hat{f}_n) - \cE(Y_n(\hat{f}_n)) \right| \le  \sup_{f \in \cF_1} \big|\cE(Y_n(f)) -\cE(f) \big| \to 0,
\]
by \eqref{emp-lim}, and
\[
\cE(Y_n(\hat{f}_n)) = \sup_{Y \in \cY_{n, r_n}} \cE(Y) \to \sup_{f \in \cF_1} \cE(f),
\]
by \eqref{E-conv}, almost surely as $n \to \infty$.
Hence, if $f_\infty = \lim_k \hat{f}_{n_k}$, by continuity of $\cE$ on $\cF^0_4$, we have 
\[
\cE( f_\infty) = \lim_k \cE(\hat{f}_{n_k}) = \sup_{f \in \cF_1} \cE(f), 
\]
and given that $f_\infty \in \cF^0_1$, we have $f_\infty \in \cS^0_1$ by definition.

The fact that $(\hat{f}_n)$ is compact with all accumulation points in $\cS^0_1$ implies that
\beq \label{accum}
\inf_{f \in \cS^0_1} \|\hat{f}_n - f\|_\infty \to 0,
\eeq
and since we have $\max_{1\leq i \leq n}\|\hat{y}_i - f(x_i)\| = \|\hat{f}_n(x_i) - f(x_i)\| \leq \|\hat{f}_n - f\|_\infty$, this immediately implies \eqref{S-conv}.
The convergence in \eqref{accum} is a consequence of the following simple result.

\begin{lem} \label{lem:infdist}
Let $(a_n)$ be a sequence in a compact metric space with metric $\delta$, that has all its accumulation points in a set $A$.  Then
\[
\inf_{a \in A} \delta(a_n, a) \to 0.
\]   
\end{lem}

\begin{proof}
If  this is not the case, then there is $\eps > 0$ such that, $\inf_{a \in A} \delta(a_n, a) \ge \eps$ for infinitely many $n$'s, denoted $n_1 < n_2 < \cdots$.  The space being compact, $(a_{n_k})$ has at least one accumulation point, which is in $A$ by assumption.  However, by construction, $(a_{n_k})$ cannot have an accumulation point in $A$.  This is a contradiction.
\end{proof}

\section{Quantitative convergence bounds} \label{sec:quant}

We obtained a general, qualitative convergence result for MVU in the preceding section and now specify some of the supporting arguments to obtain quantitative convergence speeds.  This will require some (natural) additional assumptions on $\mu$ and $\M$.  While the proof of a result like \thmref{conv} is necessarily complex, we endeavored in making it as transparent and simple as we could.  The present section is more technical, and the reader might choose to first read \secref{solution} to learn about the solutions to Continuum MVU, which imply consistency (and inconsistencies) for MVU as a dimensionality-reduction algorithm.

We consider two specific types of sets $\M$:
\bitem
\item {\em Thin sets.} $\M$ is a $d$-dimensional compact, connected, $C^2$ submanifold with $C^2$ boundary (if nonempty).  In addition, $\M \subset \M_\star$, where $\M_\star$ is a $d$-dimensional, geodesically convex $C^2$ submanifold.  
%(We assume $\M_\star$ is the smallest set such that this is true.)
\item {\em Thick sets.} $\M$ is a compact, connected subset that is the closure of its interior and has a $C^2$ boundary.  
\eitem
%By smooth we mean of bounded pointwise curvature.  
The ambient space is $\bbR^p$.  Note that our results are equally valid for piecewise smooth sets.  Thin sets are a model for noiseless data, where that the data points are sampled from a submanifold.  Note that they may have holes and boundaries.  And thick sets are a model for noisy data, where that the data points are sampled from the vicinity of a submanifold.     

An important example of thick sets are tubular neighborhoods of thin sets.  For a set $A \subset \bbR^p$ and $\eta > 0$, the $\eta$-neighborhood of $A$ is the set of points in $\bbR^p$ within Euclidean distance $\eta$ of $A$, and is denoted $B(A, \eta)$.  The reach of a set $A \subset \bbR^p$ is defined in \citep{MR0110078} as the largest $\eta$ such that, for any $x \in B(A, \eta)$ there is a unique point $a \in A$ closest to $x$.  We denote by $\rho(A)$ the reach of $A$.  Note that any thin set $A$ has positive reach, which bounds its radius of curvature from below.  While for any thick set $A$, $\partial A$ is a thin set without boundary, for any $\eta < \rho(A)$, $\bar{B}(A, \eta)$ is a thick set, with boundary having reach $\ge \rho(A) -\eta$.  

In what follows, $C$ and $C_k$ denote constants that depend only on $p$ and $d$, which may change with each appearance.  

\subsection{The regularity condition}
The first thing we do is specify the function $c$ in \eqref{reg}.  When $\M$ is a thin set, we define $r_\M = \min\big(\rho(\M_\star), \rho(\partial \M)\big)$, where by convention $\rho(\emptyset) = \infty$.  And when $\M$ is a thick set, we let $r_\M = \rho(\partial \M)$.  The following result seems valid when $r_\M = \rho(\M)$ in both cases, but the proof seems much more involved.  

\begin{lem} \label{lem:geodist}
Whether $\M$ is a thin or a thick set, \eqref{reg} is valid with 
\[
c(r) = \frac{4 r}{r_\M} \1_{\{r < r_\M/2\}} + \1_{\{r \ge r_\M/2\}}.
\]
\end{lem}

\begin{proof}
We borrow results from \citep{1349695}.  Let $x, x' \in \M$ such that $\|x - x'\| \le r_\M/2$.  

First, suppose that $\M$ is thick.  Consider the line segment joining these two points.  If this segment is included in $\M$, then $\delta_\M(x,x') = \|x - x'\|$.  Otherwise, it intersects $\partial \M$ in at least two points; among these points,  let $z$ be the closest to $x$ and $z'$ the closest to $x'$.  Since $\partial \M$ has no boundary, it is geodesically convex, so that there is a geodesic on $\partial \M$, denoted $\xi$, joining $z$ and $z'$.  \citep[Prp.~6.3]{1349695} applies since 
$\|z - z'\| \le \|x - x'\| \le r_\M/2 \le \rho(\partial \M)/2$, and $\rho(\partial \M)$ coincides with the condition number of $\partial \M$ as defined in \citep{1349695} --- and denoted by $\tau$ there.  Hence, if $\ell$ is the length of $\xi$, we have  
\beq \label{geo-bound}
\ell \le \rho(\partial \M) - \rho(\partial \M) \sqrt{1 - \frac{2 \|z - z'\|}{\rho(\partial \M)}} \le \|z - z'\| + 4 \|z - z'\|^2/r_\M,  
\eeq
using the fact that $\sqrt{1-t} \ge 1 - t/2 -t^2$ for all $t \in [0,1]$ and $r_\M \le \rho(\partial \M)$.
Let $\gamma$ be the path made of $\xi$ concatenated with the segments $[x z]$ and $[z' x']$.  If $L$ is the length of $\gamma$, we have %$L = \|x - z\| + \ell + \|z' - x'\|$, so that
\beqn
L &=& \|x - z\| + \|z' - x'\| + \ell \\
&\le& \|x - z\| + \|z' - x'\| + \|z - z'\| + 4 \|z - z'\|^2/r_\M \\
&\le& \|x - x'\| + 4 \|x - x'\|^2/r_\M,
\eeqn
using the fact that $x,z,z',x'$ are in that order on the line segment joining $x$ and $x'$.
This concludes the proof when $\M$ is thick.  

When $\M$ is thin, we distinguish two cases.  Either there is a geodesic joining $x$ and $x'$, and \citep[Prp.~6.3]{1349695} is directly applicable.  Otherwise, $\M$ is not geodesically convex.  Let $\gamma_\star$ be a geodesic on $\M_\star$ joining $x$ and $x'$.  Necessarily, it hits the boundary $\partial \M$ in at least two points.  Let $z$, $z'$, $\xi$ and $\ell$ be defined as before.  We again have \eqref{geo-bound}.  Let $(x z)_\star$ and $(z' x')_\star$ denote the arcs along $\gamma_\star$ joining $x$ and $z$, and $z'$ and $x'$, respectively.  Applying \citep[Prp.~6.3]{1349695} to each arc, which is possible since $r_\M \le \rho(\M_\star)$, we also have 
\[
\length((x z)_\star) \le \|x - z\| + 4 \|x - z\|^2/r_\M, \qquad \length((z' x')_\star) \le \|z' - x'\| + 4 \|z' - x'\|^2/r_\M.
\]
Let $\gamma$ be the curve made of concatenating these two arcs and $\xi$, and let $L$ denote its length.  We have
\beqn
L &=& \length((x z)_\star) + \length((z' x')_\star) + \ell \\
&\le& \|x - z\| + \frac{4 \|x - z\|^2}{r_\M} + \|z' - x'\| + \frac{4 \|z' - x'\|^2}{r_\M} + \|z - z'\| + \frac{4 \|z - z'\|^2}{r_\M} \\
&\le& \|x - x'\| + \frac{4 \|x - x'\|^2}{r_\M}.
\eeqn
This concludes the proof when $\M$ is thin.  
\end{proof}

\subsection{Covering numbers and a bound on the neighborhood radius} \label{sec:radius}

At what speed can we have $r_n \to 0$ and still have \eqref{omega-conv} hold?  This question is of practical importance, since the neighborhood radius may affect the output of MVU in a substantial way.    Computationally, it is preferable to have $r_n$ small, so there are fewer constraints in \eqref{Yr}.  However, we already explained that $r_n$ needs to be large enough that, at the very minimum, the resulting neighborhood graph is connected.  In fact, we required the stronger condition \eqref{omega-conv}.  

To keep the exposition simple, we assume that $\mu$ is comparable to the uniform distribution on $\M$, that is, we assume that there is a constant $\alpha >0$ such that
\beq \label{mu-lb}
\mu(B(x, \eta)) \ge \alpha \vol_{d}(B(x, \eta) \cap \M), \quad \forall x \in \M, \forall \eta > 0,
\eeq
where $\vol_d$ denotes the $d$-dimensional Hausdorff measure and $d$ denotes the Hausdorff dimension of~$\M$.
We need the following result.  Let $\omega_d$ be the volume of the $d$-dimensional unit ball.  

\begin{lem} \label{lem:vol}
Whether $\M$ is thin or thick, there is $C > 0$ such that, for any $\eta \le r_\M$ and any $x \in \M$,  
\[
\vol_d(B(x, \eta) \cap \M) \ge C \, \eta^d.
\]
\end{lem}  

\begin{proof}
It suffices to prove the result for $x \in \M \setminus \partial \M$ and for $\eta$ small enough.

{\em Thick set.}  We first assume that $\M$ is thick.  Take $x\in \M$ and $\eta < r_\M$.  If $\dist(x, \partial \M) \ge \eta$, then $B(x, \eta) \subset \M$ and the result follows immediately.  Otherwise, let $u$ be the metric projection of $x$ onto $\partial \M$, and define $z = x + (\eta/4)(x - u)/\|x - u\|$.  By the triangle inequality, $B(z, \eta/4) \subset B(x, \eta)$.  Also, by \citep[Th.~4.8]{MR0110078}, $u$ is also the metric projection of $z \in \M$ onto $\partial \M$, so that $\dist(z, \partial \M) = \|z - u\| = \|x - u\| + \eta/4 > \eta/4$.  And, necessarily, $z \in \M$, for otherwise the line segment joining $z$ to $x$ would intersect $\partial \M$, and any point on that intersection would be closer to $z$ than $u$ is, which cannot be.  Therefore, $B(z, \eta/4) \subset B(x, \eta) \cap \M$ and the result follows immediately.

{\em Thin set.}  We now assume that $\M$ is thin.  For $y \in \M$, let $T_y$ be the tangent subspace of $\M$ at $y$ and let $\pi_{y}$ denote the orthogonal projection onto $T_y$.  Because $\M$ is a $C^2$ submanifold, for every $y \in \M$, there is $\eps_y > 0$ such that $\pi_{y}$ is a $C^2$ diffeomorphism on $K_y := B(y, \eps_y) \cap \M$, with $\pi_{y}^{-1}$ being 2-Lipschitz on $\pi_y (K_y)$ --- the latter comes from the fact that $D_y \pi_y$ is the identity map and $z \to D_z \pi_y$ is continuous.  
%, with $K_y$ being diffeomorphic to either a ball or a half-ball, and with $\pi_{y}^{-1}$ being 2-Lipschitz on $\pi_y (K_y)$ --- the latter comes from the fact that $D_y \pi_y$ is the identity map and $z \to D_z \pi_y$ is continuous.  
%This implies that $\pi_y (K_y)$ contains the 
%
Since $\M$ is compact, there is $y_1, \dots, y_m \in \M$, with $m < \infty$, such that $\M \subset \cup_j B(y_j, \eps_j/2)$.  
Let $\eps = \min_j \eps_{y_j}$, which is strictly positive.  
Let $y$ be among the $y_j$'s such that $x \in B(y, \eps_j/2)$.  Assuming that $\eta < \eps/2$, we have that $B(x,\eta) \subset B(y, \eps_j)$.  
Let $U := B(y, \eps_j)$, $K = K_y$, $T = T_y$ and $\pi = \pi_y$ for short.  
%
%If $\dist(x, \partial \M) \ge \eta$, then $\pi(B(x, \eta))$ contains  

We first show that, if $\partial \M \cap K \ne \emptyset$ and $W := \pi(\partial \M \cap K)$, then $\rho(W) \ge \rho(\partial \M)$.  Indeed, for any $z, z' \in K$, we have
\[
\dist(\pi(z') - \pi(z), {\rm Tan}(W, \pi(z))) \le \dist(z' -z, {\rm Tan}(\partial \M, z)) \le \frac1{2 \rho(\partial \M)} \|z' - z\|^2, 
\]
where the first inequality follows from the facts that ${\rm Tan}(W, \pi(z)) = \pi({\rm Tan}(\partial \M, z))$ and that $\pi$ is 1-Lipschitz, and the second inequality from \citep[Th.~4.18]{MR0110078} applied to $\partial \M$.  In turn, \citep[Th.~4.17]{MR0110078} applied to $W$ implies that $\rho(W) \ge \rho(\partial \M)$.  

We can now reason as we did for thick sets, but with a twist.  To be sure, let $a = \pi(x)$ and notice that $B(a, \eta) \cap T = \pi(B(x, \eta)) \subset \pi(U)$ since $B(x, \eta) \subset U$.  
%If $\dist(x, \partial \M) \ge \eta$, then $\dist(a, W) \ge \eta/2$ since $\pi^{-1}$ is 2-Lipschitz on $\pi(K)$.  This implies that 
If $\dist(a, W) \ge \eta/2$, $B(a, \eta/2) \cap T \subset \pi(K)$.  
If $\dist(a, W) < \eta/2$, let $b$ be the metric projection of $a$ onto $W$ and define $c = a + (\eta/8)(a-b)/\|a-b\|$.
Arguing exactly as we did for thick sets, we have that $B(c, \eta/8)\cap T \subset B(a, \eta/2) \cap \pi(K)$.
Let $L = \pi^{-1}(B(c, \eta/8) \cap T)$.
Note that $L \subset \pi^{-1}(B(a, \eta/2)\cap T) \cap K \subset B(x, \eta) \cap K \subset B(x, \eta) \cap \M$, since $\pi$ is injective on $K$ and $\pi^{-1}$ is 2-Lipschitz on $\pi(K)$.  In addition, since $\pi$ is 1-Lipschitz on $K$, we have $\vol_d(L) \ge \vol_d(\pi(L)) =  \vol_d(B(c, \eta/8) \cap T)$.  This immediately implies the result.
%Fix $x \in \M$ and $\eta \in (0, r_\M)$.  We first show that there is $x' \in \M$ such that $
%B(x', \eta/2) \cap \partial \M = \emptyset$ and $B(x', \eta) \subset B(x, \eta)$.  If $\partial \M = \emptyset$, or if $\dist(x, \partial \M) > \eta$, then we may take $x' = x$.  We therefore assume that $\partial \M \ne \emptyset$ and that $\dist(x, \partial \M) \le \eta$. Let $z$ be the metric projection of $x$ onto $\partial \M$, which is uniquely defined since $\eta < r_\M \le \rho(\partial \M)$.  Let $T$ denote the tangent subspace of $\partial \M$ at $z$, and $T^\top$ its orthogonal affine subspace passing through $z$.  Let $A$ be the circle centered at $z$, with of radius $\eta/2$ and contained in the plane in $T^\top$ passing through $x$.  By \citep[Th.~4.8]{MR0110078}, $z$ is the metric projection onto $\partial \M$ of any point $a \in A$, so that $B(a, \eta/2) \cap \partial \M = \emptyset$.  (Note that the balls are open.)  We then choose $x' \in A \cap \M$, which is necessarily nonempty, for otherwise, $\M$ would be 1-dimensional at $z$.
\end{proof}

When \eqref{mu-lb} is satisfied, and $\M$ is either thin or thick, we can provide sharp rates for $r_n$.
Just as we did in \secref{packings}, we work with coverings of $\M$.     
Let $\cN(\M, \eta)$ denote the cardinality of a minimal $\eta$-covering of $\M$ for the Euclidean norm.  

\begin{lem} \label{lem:M-packing}
Suppose $\eta \le r_\M$.  When $\M$ is thick, 
\[
\cN(\M,\eta) \le C \vol_p(\M) \eta^{-p};
\]
and when $\M$ is thin and $0 \le \sigma < \rho(M)$,
\[
\cN(B(\M,\sigma), \eta) \le C \vol_d(\M) \max(\sigma, \eta)^{p-d} \eta^{-p}.
\]
The constant $C$ depends only on $p$ and $d$.
\end{lem}

\begin{proof}
Suppose $\M$ is thick and let $z_1, \dots, z_{N_\eta}$ an $\eta$-packing of $\M$ of size $N_\eta := \cN(\M,\eta)$.  Since $B(z_i,\eta/2) \cap B(z_j, \eta/2) = \emptyset$ when $i \neq j$, we have
\[
\vol_p(\M) \ge \sum_j \vol_p(B(z_j, \eta/2) \cap \M) \ge N_\eta C_p \eta^p,
\]
where $C_p$ is the constant in \lemref{vol}.  The bound on $N_\eta$ follows.  

Suppose $\M$ is thin.  When $\sigma \le \eta/4$, let $z_1, \dots, z_{N_{\eta/4}}$ an $(\eta/4)$-packing of $\M$.  Then by the triangle inequality, $B(\M, \sigma) \subset \cup_j B(z_j, \eta/2)$, and therefore $\cN(B(\M,\sigma), \eta) \le N_{\eta/4}$.  Clearly, it suffices now to focus on $\sigma \ge \eta$.  Let $z_1, \dots, z_N$ be an $(\eta/4)$-packing of $B(\M,\sigma-\eta/4)$.  Since $B(z_i,\eta/8) \cap B(z_j, \eta/8) = \emptyset$ when $i \neq j$, and $B(z_i, \eta/8) \subset B(\M,\sigma)$, we have
\[
\vol_p(B(\M,\sigma)) \ge \sum_j \vol_p(B(z_j, \eta/8)) = N \omega_p (\eta/8)^p.
\]
Hence, $N \le \omega_p^{-1} (\eta/8)^{-p} \vol_p(B(\M,\sigma))$.  By Weyl's volume formula for tubes \citep{MR1507388}, we have $\vol_p(B(\M,\sigma)) \le C_1 \vol_d(\M) \sigma^{p-d}$ for a constant $C_1$ depending on $p$ and $d$.  Since we have $B(\M,\sigma) \subset \cup_j B(z_j, \eta/2)$, we have $\cN(B(\M,\sigma), \eta) \le N$, and the result follows.  
\end{proof}

We are now ready to take a closer look at \eqref{omega-conv}.  Let $\eta_n$ be defined as in \secref{packings}.  By \eqref{mu-lb} and \lemref{vol}, we have $p_\eta \ge C_1 \alpha \eta^{d}$, and we have $\cN(\M,\eta) \le C_2 \eta^{-d}$ by \lemref{M-packing}, where $C_1$ and $C_2$ depend only on $\M$.  Hence,
\[
\cN(\M,\eta) (1 - p_\eta)^n \le C_2 \eta^{-d} \big(1 - C_1 \alpha \eta^{d}\big)^n \le C_2 \eta^{-d} e^{-n C_1 \alpha \eta^{d}} \le \frac1{n^2},
\]
when
\[
\eta^d \ge (C_1 \alpha \, n)^{-1} \log \big(C_2 \eta^{-d} n^2 \big).
\]
We deduce that any $r_n \gg r_n^\dag := (\log(n)/n)^{1/d}$ satisfies \eqref{omega-conv} with any $\lambda_n \to 0$ such that $\lambda_n \gg r_n^\dag/r_n$.

\subsection{Packing numbers of Lipschitz functions on $\M$} \label{sec:LipM}
It appears necessary to provide a bound for $\cN_\infty(\cF_1^0, \eta)$.  For this, we follow the seminal work of \cite{MR0124720} on entropy bounds for classical functions classes (including Lipschitz classes).  We provide details for completeness.

\begin{lem} \label{lem:F-packing}
For any $\M$ compact, connected subset of $\bbR^p$ satisfying \eqref{reg}, there is a constant $C$ such that  
\[
\log \cN_\infty(\cF_1^0, \eta) \le C \, (\log(1/\eta) + \cN(\M, \eta/C)),
\]
for all $0<\eta\le 1$.
\end{lem}

In particular, if $\M$ is thin or thick, we have $\log \cN_\infty(\cF_1^0,\eta) \le C \eta^{-d}$ by \lemref{M-packing} and \lemref{F-packing}.  

\begin{proof}
Take $0 < \eps \le 1/\sqrt{p}$ and let $C_0 = 2 \sqrt{p} (2+c(2))  $.  For $j = (j_1, \dots, j_p) \in \bbZ^p$, let $Q_j = \prod_{s=1}^p [j_s \, \eps, (j_s+1) \eps)$.  Let $J = \{j : Q_j \cap \M \neq \emptyset\}$, which we see as a subgraph of the lattice for the $2^p$-nearest neighbor topology.

Note that $|J| \le C_1 \cN(\M, \eps)$.  
Indeed, let $e_1, \dots, e_{2^p}$ be the vertices of the unit hypercube of $\bbR^p$ and let $Z_s = e_s + (2 \bbZ)^p$.
Also, let $Z_0 = (2 \bbZ)^p$.  By construction, $Z_1, \dots, Z_{2^p}$ is a partition of $\bbZ^p$.
Therefore, there is $s$ (say $s = 1$) such that $|J \cap Z_s| \ge |J|/2^p$.
For each $j\in J \cap Z_1$, pick $x_j \in Q_j \cap \M$.
By construction, for any $j \ne j'$ both in $J \cap Z_1$, $\|x_j-x_{j'}\| > 2\eps$, so $|J\cap Z_1|$ is smaller than the $2\eps$-packing number of $M$, which is smaller than the $\eps$-covering number of $\M$.

%$B(x_j, \eps) \cap B(x_{j'}, \eps) = \emptyset$ for $j \ne j'$ both in $J \cap Z_1$.

%And by \lemref{vol}, there is a constant $C$ such that $\vol_d(B(x_j, \eps) \cap \M) \ge C \eps^{-d}$.  
%Hence, we have 
%\[
%\vol_d(\M) \ge \sum_{j \in J \cap Z_1} \vol_d(B(x_j, \eps) \cap \M) \ge |J \cap Z_1| \cdot C \eps^{-d} \ge  \frac{C}{2^p} \eps^{-d} |J|,
%\]
%and from this the bound on $|J|$ follows.
Note also that $\cup_j Q_j$ is connected because $\M$ is.  Let $\pi_1, \dots, \pi_\ell$ be a sequence covering $J$ and such that $Q_{\pi_s}$ and $Q_{\pi_{s-1}}$ are adjacent.  A depth-first construction gives a sequence $\pi$ of length at most $\ell \le C_2 |J|$, since each $Q_j$ has a constant number ($=2^p$) of adjacent hypercubes.

Let $y_1, \dots, y_{m}$ be an enumeration of the $\eps$-grid $(\eps \bbZ \cap [-\diam(\M), \diam(\M)])^p$.  Note that $m \le C_3 \eps^{-p}$ and that, for each $s$ there are at most $C_4$ indices $t$ such that $\|y_s - y_t\| \le C_0 \eps$.  

Consider the class $\cG$ of piecewise-constant functions $g : \M \to \bbR^p$ of the form $g(x) = y_{t_j}$ for all $x \in Q_j\cap M$ and such that $\|y_{t_j} - y_{t_k}\| \le C_0 \eps$ when $Q_j$ and $Q_k$ are adjacent.  This is a subclass of the class of functions of the form $g(x) = y_{t_{\pi(j)}}$ for all $x \in Q_{\pi(j)}$ and such that $\|y_{t_{\pi(j)}} - y_{t_{\pi(j-1)}}\| \le C_0 \eps$.  The cardinality of the larger class is at most $m C_4^{\ell-1}$, since there are $m$ possible values for $y_{t_{\pi(1)}}$ and then, at each step along $\pi$, there at most $C_4$ choices.  Therefore, 
\beqn
\log |\cG| &\le& \log m + \ell \log C_4 \\
&\le& \log (C_3) + p \log(1/\eps) + C_2 C_1  \cN(\M, \eps) \log(C_4)\\
&\le& C_5 (\log(1/\eps) + \cN(\M, \eps)). 
\eeqn

For each $j$, choose $z_j \in Q_j \cap \M$.  Take any $f \in \cF_1^0$.  For each $j$, let $t_j$ be such that $\|f(z_j) - y_{t_j}\| \le \sqrt{p}\eps$ and let $g$ be defined by $g(x) = y_{t_j}$ for all $x \in Q_j$.  Suppose $Q_j$ and $Q_k$ are adjacent, so that $\|z_j - z_k\| \le 2 \sqrt{p} \eps \le 2$.  By the triangle inequality, \eqref{Lip} and \eqref{reg},  we have
\beqn
\|y_{t_j} - y_{t_k}\| &\le& \|f(z_j) - f(z_k)\| + \|y_{t_j} - f(z_j)\| + \|y_{t_k} - f(z_k)\| \\
&\le& (1+c(\|z_j - z_k\|)) \|z_j - z_k\| + \sqrt{p}\eps + \sqrt{p}\eps \\
& \le & (1+c(2))2\sqrt{p}\eps  + 2\sqrt{p}\eps \\
& = & C_0 \eps.
\eeqn
so that $g \in \cG$.  Moreover, for $x \in Q_j \cap \M$, 
\[
\|g(x) - f(x)\| = \|y_{t_j} - f(z_j)\| + \|f(z_j) - f(x)\| \le \sqrt{p}\eps + (1 + c(\sqrt{p}\eps))\sqrt{p} \eps \leq (2 + c(1))\sqrt{p} \eps.
\]
The result follows from choosing $\eps = \eta/((2 + c(1))\sqrt{p})$.
\end{proof}

\subsection{Quantitative convergence bound}

From \eqref{ld} and \lemref{F-packing}, there is a constant $C>0$ such that 
\[
\P\left(\sup_{f \in \cF_1} \big| \cE(Y_n(f)) - \cE(f) \big| > C n^{-1/(d+2)} \right) \le \exp(- n^{-(d+1)/(d+2)}).
\]
Using this fact in \eqref{E-diff-conv}, together with \lemref{geodist} and the order of magnitude for $r_n$ derived in \secref{radius}, leads to a bound on the rate of convergence in \eqref{E-conv} via the Borel-Cantelli Lemma.  

\begin{thm} \label{thm:quant} 
Suppose that $\M$ is either thin or thick, of dimension $d$, and that \eqref{mu-lb} holds.  Assume that $r_n \to 0$ such that $r_n \gg r_n^\dag := (\log(n)/(\alpha \, n))^{1/d}$ and take any $a_n \to \infty$.  Then, with probability one, 
\[
\big| \sup\{\cE(Y) : Y \in \cY_{n, r_n}\} - \sup\{\cE(f) : f \in \cF_1\} \big| \le a_n \big(r_n + \frac{r_n^\dag}{r_n} + n^{-1/(2+d)} \big),
\]
for $n$ large enough.
\end{thm}

Unfortunately, we do not have a quantitative bound on the rate of convergence of the solutions in \eqref{S-conv}.

\section{Continuum MVU} \label{sec:solution}

Now that we established the convergence of Discrete MVU to Continuum MVU, we study the latter, and in particular its solutions.  We mostly focus on the case where $\M$ is isometric to a Euclidean domain.  

\medskip
\noindent {\bf Isometry assumption.}  We assume that $\M$ is isometric to a compact, connected domain $\D \subset \bbR^d$.  Specifically, there is a bijection $\psi: \M \to \D$ satisfying $\delta_\D(\psi(x), \psi(x')) = \delta_\M(x,x')$ for all $x, x' \in \M$.

\medskip
As a glimpse of the complexity of the notion of isometry, and also for further reference, consider a domain $\D$ as above.  Then the canonical inclusion $\iota$ of $\D$ in $\bbR^d$ is not necessarily an isometry between the metric spaces $(\D,\delta_\D)$ and $(\bbR^d, \|\cdot\|)$.  To see this, let $x$ and $x'$ be two points of $\D$.
Let $\gamma$ be a shortest path connecting $x$ to $x'$ in $\D$.
%The existence of this shortest path follows from the extension of the classical Hopf-Rinow theorem of Riemannian geometry; see e.g. Theorem 2.5.28 in \cite{MR1835418}.
Suppose that $\iota: (\D,\delta_\D) \to (\bbR^d, \|\cdot\|)$ is an isometry.
Then, $L(\iota\circ\gamma) = L(\gamma) =  \delta_{\D}(x,x') = \|\iota(x) - \iota(x')\|$.
So the image path $\iota\circ\gamma$ is a shortest path connecting $\iota(x)$ to $\iota(x')$, hence a segment.
Since this segment lies in $\iota(\D)=\D$, and since this holds for any pair of points $x,x'$ in $\D$, this implies that $\D$ is convex.
Conversely, if $\D$ is convex, the canonical inclusion $\iota$ is an isometry.
%In conclusion, we have the following.
%
%\begin{lem} \label{lem:iso-convex}
%Let $\D$ be a compact, arc-wise connected subset of $\bbR^d$.  Then the canonical inclusion $\iota: (\D,\delta_\D) \to (\bbR^d, \|\cdot\|)$ is an isometry if, and only if, $\D$ is convex.
%\end{lem}

%We start with a general result that says that MVU flattens $\M$ when it is possible to do so.  We then improve on this result and 
We start by showing that, in the case where $\M$ is isometric to a convex domain, then MVU recovers this convex domain modulo a rigid transformation, so that MVU is consistent is that case.  The last part of the section is dedicated to a perturbation analysis that shows two things.  First, that Continuum MVU changes slowly with the amount of noise, up to a point.  And second, that when $\M$ is isometric to a domain that is not convex, MVU may not recover this domain.  We provide some illustrative examples of that. 

In the following, we identify $\bbR^d$ with $\bbR^d \times \{0\}^{p-d} \subset \bbR^p$.

%
%\subsection{MVU's flattening property}
%
%\eac{Hoping we can prove something like that...}
%
%
\subsection{Consistency under the convex assumption}

If we assume that $\D$ is convex, then MVU recovers $\D$ up to a rigid transformation, in the following sense.  Recall that $\mathcal{S}_1$ is the solution space of Continuum MVU.  

\begin{thm}
\label{thm:convex}
Suppose that $\M$ is isometric to a convex subset $\D \subset \bbR^d$ with isometry mapping $\psi: \M\to\D$, and that \eqref{mu-lb} holds.
Then
$$\cS_1 = \{\zeta\circ\psi\,:\,\zeta\in{\rm Isom}(\bbR^p)  \}.$$
\end{thm}

\begin{proof}
Note first that, since $\D$ is convex, its intrinsic distance coincides with the Euclidean distance of $\bbR^d$, i.e., $\delta_D = \|\cdot\|$.
For all $f$ in $\cF_1$, we have
\begin{eqnarray*}
\mathcal{E}(f) & = & \int_{M\times M} \|f(x)-f(x')\|^2\mu({\rm d}x)\mu({\rm d}x')  \\
 & \leq &  \int_{M\times M} \delta_\M(x,x')^2\mu({\rm d}x)\mu({\rm d}x') \\
 & = &  \int_{M\times M} \delta_D(\psi(x),\psi(x'))^2\mu({\rm d}x)\mu({\rm d}x') \\
 & = &  \int_{M\times M} \|\psi(x) - \psi(x')\|^2\mu({\rm d}x)\mu({\rm d}x') \\
 & = & \int_{\D\times\D} \|z - z'\|^2 (\mu\circ\psi^{-1})({\rm d}z)(\mu\circ\psi^{-1})({\rm d}z'),
 \end{eqnarray*}
while
$$\mathcal{E}(\psi) = \int_{\D\times\D} \|z - z'\|^2 (\mu\circ\psi^{-1})({\rm d}z)(\mu\circ\psi^{-1})({\rm d}z').$$
So 
$$\sup_{f\in\mathcal{F}_1} \cE(f) = \mathcal{E}(\psi) = \int_{\D\times\D} \|z - z'\|^2 (\mu\circ\psi^{-1})({\rm d}z)(\mu\circ\psi^{-1})({\rm d}z').$$
Hence $\psi \in \cS_1$, and since $\mathcal{E}(\zeta\circ\psi) = \mathcal{E}(\psi)$ for any isometry $\zeta:\bbR^p\to\bbR^p$,
$$\{\zeta\circ\psi\,:\,\zeta\in{\rm Isom}(\bbR^p\} \subset \mathcal{S}_1.$$

Now let $f:M\to\bbR^p$ be a function in $\mathcal{F}_1$ so that $\|f(x)-f(x')\| \leq \delta_\M(x,x')$ for any points $x$ and $x'$ in $\M$.
Suppose that $f$ is not an isometry.
Then there exists two points $x$ and $x'$ in $\M$ such that 
$$\|f(x)-f(x')\| < \delta_\M(x,x').$$
By continuity of $f$, there exists a nonempty open subset $U$ of $M\times M$ containing $(x,x')$ such that $\|f(z)-f(z')\| < \delta_\M(z,z')$ for all $(z,z')$ in $U$.
In addition, $\mu(U)>0$ by \eqref{mu-lb}. 
Consequently
\begin{eqnarray*}
\cE(f) & = & \int_{M\times M \setminus U} \|f(x)-f(x')\|^2 \mu({\rm d}x)\mu({\rm d}x') + \int_U\|f(x)-f(x')\|^2 \mu({\rm d}x)\mu({\rm d}x')\\
& < &\int_{M\times M} \delta_\M(x,x')^2 \mu({\rm d}x)\mu({\rm d}x') \\
& = & \sup_{f\in \mathcal{F}_1} \mathcal{E}(f).
\end{eqnarray*}
So any function $f$ in $\cF_1$ which is not an isometry onto its image does not belong to $S_1$.

At last, since for any isometry $f$ in $\mathcal{S}_1$, the map $f\circ\psi^{-1}:\bbR^p\to\bbR^p$ is an isometry, there exists some isometry $\zeta \in {\rm Isom}(\bbR^p)$ such that $f=\zeta\circ \psi$, and we conclude that
$$\{\zeta\circ\psi\,:\,\zeta\in{\rm Isom}(\bbR^p)\} = \mathcal{S}_1.$$
\end{proof}

In conclusion, MVU recovers the isometry when the domain $\D$ is convex.  Note that this is also the case of ISOMAP.

\subsection{Noisy setting} \label{sec:noise}

When the setting is noisy, with noise level $\sigma \ge 0$, $x_1, \dots, x_n$ are sampled from $\mu_\sigma$, a (Borel) probability distribution on $\bbR^p$ with support $\M_\sigma := \bar{B}(\M, \sigma)$, i.e., $M_\sigma$ is composed of all the points of $\mathbb{R}^p$ that are at a distance at most $\sigma$ from $M$. 
%The only regularity condition we require is the existence, for some $\sigma_0 > 0$, of a measurable map $\pi: \M_{\sigma_0} \to \M$ such that $\pi(x) \in \argmin\{\|x - x'\|: x' \in \M\}$.  This is the case when $M$ is of positive reach $\rho(\M) > 0$ and $\sigma_0 < \rho(\M)$.
To speak of noise stability, we assume that $\mu_\sigma$ converges weakly when $\sigma \to 0$.  Let $\cF_{1,\sigma}$ denote the class of 1-Lipschitz functions on $\M_\sigma$, and so on.  Our simple perturbation analysis is plainly based on the fact that $\cE$ is continuous with respect to the noise level, in the following sense.  This immediately implies that MVU is tolerant to noise. 

\begin{lem}  \label{lem:noise}
Let $\M \subset \bbR^p$ be of positive reach $\rho(\M) > 0$ and assume that $\mu_\sigma \to \mu_0$ weakly when $\sigma \to 0$.  Then as $\sigma \to 0$, we have 
\beq \label{E-noise-conv}
\sup_{f \in \cF_{1,\sigma}} \cE_\sigma(f) \to \sup_{f \in \cF_1} \cE(f), 
\eeq
and
\beq \label{S-noise-conv}
\sup_{f \in \cS_{1,\sigma}} \ \inf_{g \in \cS_1} \ \sup_{x \in \M_\sigma} \ \inf_{z \in \M} \|f(x) - g(z)\| \to 0.
\eeq
\end{lem}

\begin{proof}
The metric projection $\pi: B(\M, \rho(\M)) \to \M$ with $\pi(x) = \argmin\{\|x - x'\|: x' \in \M\}$, is well-defined and 1-Lipschitz \citep[Th.~4.8]{MR0110078}.  

Consider any sequence $\sigma_m \to 0$ with $\sigma_m < \rho(\M)$ for all $m \ge 1$, and let $f_m \in \cS^0_{1,\sigma_m}$.  Let $g_m$ denote the restriction of $f_m$ to $\M$.  Since $(g_m) \subset \cF^0_1$ and $ \cF^0_1$ is compact for the supnorm, it admits a convergent subsequence.  Assume $(g_m)$ itself is convergent, without loss of generality.  Then $g_{m} \to g_\star$, with $g_\star \in \cF_1^0$.  For $x \in B(\M, \rho(\M))$, define $f_\star(x) = g_\star(\pi(x))$.  Then for $x \in \M_{\sigma_m}$, we have
\beqn
\|f_\star(x) - f_{m}(x)\| &\le& \|g_\star(\pi(x)) - g_{m}(\pi(x))\| + \|f_{m}(\pi(x)) - f_{m}(x)\| \\
&\le& \|g_\star - g_{m}\|_\infty + \|\pi(x) - x\| \\
&\le& \|g_\star - g_{m}\|_\infty + \sigma_m,
\eeqn
since $f_{m} \in \cF_{1,\sigma_m}$ and the segment $[\pi(x), x] \subset \M_{\sigma_m}$.  The latter is due to $\|\pi(x) - x\| \le \sigma_m$ and $B(\pi(x), \sigma_m) \subset \M_{\sigma_m}$, both by definition.  Hence, as functions on $\M_{\sigma_m}$, we have $\|f_\star(x) - f_{m}(x)\|_\infty \to 0$, i.e.,
$$\sup_{x\in \M_{\sigma_m}} \|f_\star(x) - f_{m}(x)\| \to 0.$$
By \eqref{E-lip}, again applied to functions on $\M_{\sigma_m}$ for a fixed $m$, we have
\begin{eqnarray*}
\big|\cE_{\sigma_m}(f_m) - \cE_{\sigma_m}(f_\star)\big| & \leq & 4\|f_\star(x) - f_{m}(x)\|_\infty {\rm diam}(M_{\sigma_m}) \\
& \leq & 4\|f_\star(x) - f_{m}(x)\|_\infty {\rm diam}(B(M,\rho(M)))\\
&  \to & 0,
\end{eqnarray*}
and since $f_\star$ does not depend on $m$ and is bounded, we also have
\beq \label{noise-proof1}
\cE_{\sigma_m}(f_\star) \to \cE(f_\star) = \cE(g_\star) \le \sup_{\cF_1} \cE.
\eeq
Hence
\begin{eqnarray*}
\sup_{\cF_{1,\sigma_m}} \cE_{\sigma_m}   & = & \cE_{\sigma_m}(f_m) \\
& = & \cE(f_\star) + \cE_{\sigma_m}(f_\star)  - \cE(f_\star) + \cE_{\sigma_m}(f_m) -  \cE_{\sigma_m}(f_\star)\\
& \leq & \sup_{\cF_1} \cE + \cE_{\sigma_m}(f_\star)  - \cE(f_\star) + \cE_{\sigma_m}(f_m) -  \cE_{\sigma_m}(f_\star),
\end{eqnarray*}
and we deduce that 
\[
\varlimsup_{m \to \infty} \sup_{\cF_{1,\sigma_m}} \cE_{\sigma_m} \le \sup_{\cF_1} \cE,
\]
and since this is true for all sequences $\sigma_m \to 0$ (and $m$ large enough), we have 
\[
\varlimsup_{\sigma \to 0} \sup_{\cF_{1,\sigma}} \cE_{\sigma} \le \sup_{\cF_1} \cE.
\]

For the reverse relation, choose $g \in \cS_1$ and for $x \in B(\M, \rho(\M))$ define $f(x) = g(\pi(x))$.
As above, let $\sigma_m\to 0$ with $\sigma_m \leq \rho(M)$.
Then $f \in \cF_{1, \sigma_m}$ by composition, so that
\[
\cE_{\sigma_m}(f) \le \sup_{\cF_{1,\sigma_m}} \cE_{\sigma_m}.
\]
On the other hand, 
\[
\cE_{\sigma_m}(f) \to \cE(f) = \cE(g) = \sup_{\cF_1} \cE.
\]
Hence,
\[
\sup_{\cF_1} \cE \le \varlimsup_{\sigma \to 0} \sup_{\cF_{1,\sigma}} \cE_{\sigma}.
\]
This concludes the proof of \eqref{E-noise-conv}.

Equation \eqref{S-noise-conv} is now proved based on \eqref{E-noise-conv} in the same way \eqref{S-conv} is proved based on \eqref{E-conv}, by contradiction.
To be sure, assume \eqref{S-noise-conv} is not true.  Then it is also not true for $\cS^0_{1,\sigma}$ and $\cS^0_1$.  Hence, there is $\eps > 0$, a sequence $\sigma_m \to 0$ and $f_m \in \cS^0_{1,\sigma_m}$ such that 
\[
\inf_{g \in \cS^0_1} \ \sup_{x \in \M_{\sigma_m}} \ \inf_{z \in \M} \|f_m(x) - g(z)\| \ge \eps,
\]
for infinitely many $m$'s.  Without loss of generality, we assume this is true for all $m$.  For each $m$, let $g_m$ be the restriction of $f_m$ to $\M$.  Then, taking a subsequence if needed, $g_m \to g_\star\in \cF^0_1$ in supnorm.  As before, define $f_\star(x) =  g_\star(\pi(x))$ for $x \in B(\M, \rho(\M))$.  Following the same arguments, we have 
\[
\sup_{x \in \M_{\sigma_m}} \|f_\star(x) - f_m(x)\| \to 0.
\]
We also see that, necessarily, $g_\star \in S^0_1$, for otherwise the inequality in \eqref{noise-proof1} would be strict and this would imply that \eqref{E-noise-conv} does not hold.  
Hence
\[
\sup_{x \in \M_{\sigma_m}} \|f_\star(x) - f_m(x)\| \ge \sup_{x \in \M_{\sigma_m}} \ \inf_{z \in \M} \|f_m(x) - g_\star(z)\| \ge \inf_{g \in \cS^0_1} \sup_{x \in \M_{\sigma_m}} \ \inf_{z \in \M} \|f_m(x) - g(z)\|.
\]
This leads to a contradiction.  Hence the proof of \eqref{S-noise-conv} is complete.
\end{proof}

\subsection{Inconsistencies} \label{sec:inconsistent}

We provide two emblematic situations where MVU fails to recover $\D$.  They are both consequences of MVU's robustness to noise.  In both cases, we consider the simplest situation where $\M = \D \subset \bbR^2$ and $\mu$ is the uniform distribution.  Note that $\psi$ is the identity function in this case, i.e., $\psi(x) = x$, and the Isometry Assumption is clearly satisfied. We use the same notation as in \secref{noise} and let $\mu_\sigma$ denote the uniform distribution on $\M_\sigma$.  

\bigskip
\noindent {\bf Nonconvex without holes.}  Suppose $\M_0 \subset \bbR^2$ is a curve homeomorphic to a line segment, but different from a line segment, and for $\sigma > 0$, let $\M_\sigma$ be the (closed) $\sigma$-neighborhood of $\M_0$.  We show that there is a numeric constant $\sigma_0 > 0$ such that, when $\sigma < \sigma_0$, $\psi$ does not maximize the energy $\cE_\sigma$.  To see this, we utilize \lemref{noise} to assert that $\cS_{1, \sigma} \to \cS_{1,0}$ in the sense of \eqref{S-noise-conv}, and that $\psi \notin \cS_{1,0}$, because $\cS_{1,0}$ is made of all the functions that map $\M$ to a line segment isometrically.  So there is $\sigma_0 > 0$ such that $\psi \notin \cS_{1, \sigma}$ for all $\sigma < \sigma_0$.  This also implies that no rigid transformation of $\bbR^2$ is part of $\cS_{1, \sigma}$.  If we now let $\D = \M = \M_\sigma$ for some $0 < \sigma < \sigma_0$, we see that we do not recover $\D$ up to a rigid transformation.

\bigskip
\noindent {\bf Convex boundary and convex hole.}  Let $K_a$ denote the axis-aligned ellipse of $\bbR^2$ with semi-major axis length equal to $a$ and perimeter equal to $2\pi$.  Note that, necessarily, $1 \le a < \pi/2$, with the extreme cases being the unit circle ($a = 1$) and the interval $[-\pi/2,\pi/2]$ swept twice ($a = \pi/2$).  Denote by $b = b(a)$ the semi-minor axis length of $K_a$, implicitly defined by
\[
\int_0^{2\pi} \sqrt{a^2 \sin^2 t + b^2 \cos^2 t} \, {\rm d}t = 2\pi.
\]
We have
\[
F(a) := \int_{K_a} \|x\|^2 {\rm d}x = \int_0^{2\pi} \big(a^2 \cos^2 t + b^2 \sin^2 t\big) \sqrt{a^2 \sin^2 t + b^2 \cos^2 t} \, {\rm d}t. 
\] 
This daunting expression is much simplified when $a = 1$, in which case it is equal to $2\pi$, and when $a = \pi/2$, in which case it is equal to $\pi^2/12$.  Since the former is larger than the latter, and $F$ is continuous in $a$, there is $a_\star$ such that, for $a > a_\star$, $F(a) < F(1)$.  (We actually believe that $a_\star = 1$.)  

Fix $a \in (a_\star, \pi/2)$ and let $\M_0 = K_a = \phi^{-1}(K_1)$, where $\phi: \bbR^2 \to \bbR^2$ sends $x = (x_1, x_2)$ to $\phi(x) = (x_1/a,  x_2/b)$.  Note that $K_1$ is the unit circle.  
By the previous calculations and our choice for $a$, the identity function $\psi$ is not part of $\cS_{1,0}$, since
\[
\cE_0(\psi) = \frac1\pi \int_{\M_0} \|x\|^2 {\rm d}x = \frac1\pi F(a) < \frac1\pi F(1) = 2 = \frac1\pi \int_{\M_0} \|\phi(x)\|^2 {\rm d}x = \cE_0(\phi). 
\]  
As before, let $\M_\sigma$ be the (closed) $\sigma$-neighborhood of $\M_0$.  Again, there is a numeric constant $\sigma_0 > 0$ such that, when $\sigma < \sigma_0$, $\psi$ does not maximize the energy $\cE_\sigma$, and we conclude again that if $\D = \M = \M_\sigma$, MVU does not recover $\D$ up to a rigid transformation.

\section{Discussion} \label{sec:discussion}

We leave behind a few interesting problems.
\bitem
\item {\em Convergence rate for the solution(s).}  We obtained a convergence rate for the energy in \thmref{quant}, but no corresponding result for the solution(s).  Such a result necessitates a fine examination of the speed at which the energy decreases near the space of maximizing functions.  

\item {\em Flattening property of MVU.}  Assume that $\M$ satisfies the Isometry Assumption.  Though we showed that MVU is not always consistent in the sense that it may not recover the domain $\D$ up to a rigid transformation, we believe that MVU always flattens the manifold $\M$ in this case, meaning that it returns a set $\S$ which is a subset of some $d$-dimensional affine subspace.  If this were true, it would make MVU consistent in terms of dimensionality reduction!

\item {\em Solution space in general.}  As pointed out by \cite{pap}, and as we showed in \thmref{conv}, characterizing the solutions to Continuum MVU is crucial to understanding the behavior of Discrete MVU.  In \thmref{convex}, we worked out the case where $\M$ is isometric to a convex set.  What can we say when $\M$ is isometric to a sphere?  Is MVU able to recover this isometry?  This question is non-trivial even when $\M$ is isometric to a circle.  In fact, showing that the energy over ellipses (of same perimeter) is maximized for a circle is not straightforward, as seen in \secref{inconsistent}.
\eitem

\subsection*{Acknowledgements}
%We would like to thank XXX.
This work was partially supported by a grant from the National Science Foundation (NSF DMS 0915160) and by a grant from the French National Research Agency (ANR 09-BLAN-0051-01).

% REFERENCES
\bibliographystyle{chicago}
\bibliography{mvubib}

\begin{thebibliography}{}

\bibitem[\protect\citeauthoryear{Belkin and Niyogi}{Belkin and
  Niyogi}{2003}]{Belkin03}
Belkin, M. and P.~Niyogi (2003).
\newblock Laplacian eigenmaps for dimensionality reduction and data
  representation.
\newblock {\em Neural Computation\/}~{\em 15\/}(16), 1373--1396.

\bibitem[\protect\citeauthoryear{Belkin and Niyogi}{Belkin and
  Niyogi}{2005}]{belkin2005towards}
Belkin, M. and P.~Niyogi (2005).
\newblock Towards a theoretical foundation for laplacian-based manifold
  methods.
\newblock In P.~Auer and R.~Meir (Eds.), {\em Learning Theory}, Volume 3559 of
  {\em Lecture Notes in Computer Science}, pp.\  835--851. Springer Berlin /
  Heidelberg.

\bibitem[\protect\citeauthoryear{Bernstein, De~Silva, Langford, and
  Tenenbaum}{Bernstein et~al.}{2000}]{bernstein2000graph}
Bernstein, M., V.~De~Silva, J.~Langford, and J.~Tenenbaum (2000).
\newblock Graph approximations to geodesics on embedded manifolds.
\newblock Technical report, Technical report, Department of Psychology,
  Stanford University.

\bibitem[\protect\citeauthoryear{Brand}{Brand}{2003}]{brand2003charting}
Brand, M. (2003).
\newblock Charting a manifold.
\newblock {\em Advances in neural information processing systems\/}, 985--992.

\bibitem[\protect\citeauthoryear{Brudnyi and Brudnyi}{Brudnyi and
  Brudnyi}{2012}]{MR2882877}
Brudnyi, A. and Y.~Brudnyi (2012).
\newblock {\em Methods of geometric analysis in extension and trace problems.
  {V}olume 1}, Volume 102 of {\em Monographs in Mathematics}.
\newblock Birkh\"auser/Springer Basel AG, Basel.

\bibitem[\protect\citeauthoryear{Burago, Burago, and Ivanov}{Burago
  et~al.}{2001}]{MR1835418}
Burago, D., Y.~Burago, and S.~Ivanov (2001).
\newblock {\em A course in metric geometry}, Volume~33 of {\em Graduate Studies
  in Mathematics}.
\newblock Providence, RI: American Mathematical Society.

\bibitem[\protect\citeauthoryear{Coifman and Lafon}{Coifman and
  Lafon}{2006}]{coifman2006diffusion}
Coifman, R. and S.~Lafon (2006).
\newblock Diffusion maps.
\newblock {\em Applied and Computational Harmonic Analysis\/}~{\em 21\/}(1),
  5--30.

\bibitem[\protect\citeauthoryear{de~la Pe{\~n}a and Gin{\'e}}{de~la Pe{\~n}a
  and Gin{\'e}}{1999}]{MR1666908}
de~la Pe{\~n}a, V.~H. and E.~Gin{\'e} (1999).
\newblock {\em Decoupling}.
\newblock Probability and its Applications (New York). New York:
  Springer-Verlag.
\newblock From dependence to independence, Randomly stopped processes.
  $U$-statistics and processes. Martingales and beyond.

\bibitem[\protect\citeauthoryear{Donoho and Grimes}{Donoho and
  Grimes}{2003}]{Donoho03Hessian}
Donoho, D. and C.~Grimes (2003).
\newblock Hessian eigenmaps: {L}ocally linear embedding techniques for
  high-dimensional data.
\newblock {\em P. Natl. Acad. Sci. USA\/}~{\em 100\/}(10), 5591--5596.

\bibitem[\protect\citeauthoryear{Federer}{Federer}{1959}]{MR0110078}
Federer, H. (1959).
\newblock Curvature measures.
\newblock {\em Trans. Amer. Math. Soc.\/}~{\em 93}, 418--491.

\bibitem[\protect\citeauthoryear{Gin{\'e} and Koltchinskii}{Gin{\'e} and
  Koltchinskii}{2006}]{MR2387773}
Gin{\'e}, E. and V.~Koltchinskii (2006).
\newblock Empirical graph {L}aplacian approximation of {L}aplace-{B}eltrami
  operators: large sample results.
\newblock In {\em High dimensional probability}, Volume~51 of {\em IMS Lecture
  Notes Monogr. Ser.}, pp.\  238--259. Beachwood, OH: Inst. Math. Statist.

\bibitem[\protect\citeauthoryear{Goldberg, Zakai, Kushnir, and Ritov}{Goldberg
  et~al.}{2008}]{MR2438829}
Goldberg, Y., A.~Zakai, D.~Kushnir, and Y.~Ritov (2008).
\newblock Manifold learning: the price of normalization.
\newblock {\em J. Mach. Learn. Res.\/}~{\em 9}, 1909--1939.

\bibitem[\protect\citeauthoryear{Hein, Audibert, and von Luxburg}{Hein
  et~al.}{2005}]{hein2005graphs}
Hein, M., J.-Y. Audibert, and U.~von Luxburg (2005).
\newblock From graphs to manifolds -- weak and strong pointwise consistency of
  graph laplacians.
\newblock In P.~Auer and R.~Meir (Eds.), {\em Learning Theory}, Volume 3559 of
  {\em Lecture Notes in Computer Science}, pp.\  470--485. Springer Berlin /
  Heidelberg.

\bibitem[\protect\citeauthoryear{Hoeffding}{Hoeffding}{1963}]{hoeffding}
Hoeffding, W. (1963).
\newblock Probability inequalities for sums of bounded random variables.
\newblock {\em J. Amer. Statist. Assoc.\/}~{\em 58}, 13--30.

\bibitem[\protect\citeauthoryear{Kolmogorov and Tikhomirov}{Kolmogorov and
  Tikhomirov}{1961}]{MR0124720}
Kolmogorov, A.~N. and V.~M. Tikhomirov (1961).
\newblock {$\varepsilon $}-entropy and {$\varepsilon $}-capacity of sets in
  functional space.
\newblock {\em Amer. Math. Soc. Transl. (2)\/}~{\em 17}, 277--364.

\bibitem[\protect\citeauthoryear{Lang and Schroeder}{Lang and
  Schroeder}{1997}]{MR1466337}
Lang, U. and V.~Schroeder (1997).
\newblock Kirszbraun's theorem and metric spaces of bounded curvature.
\newblock {\em Geom. Funct. Anal.\/}~{\em 7\/}(3), 535--560.

\bibitem[\protect\citeauthoryear{Niyogi, Smale, and Weinberger}{Niyogi
  et~al.}{2008}]{1349695}
Niyogi, P., S.~Smale, and S.~Weinberger (2008).
\newblock Finding the homology of submanifolds with high confidence from random
  samples.
\newblock {\em Discrete Comput. Geom.\/}~{\em 39\/}(1), 419--441.

\bibitem[\protect\citeauthoryear{Paprotny and Garcke}{Paprotny and
  Garcke}{2012}]{pap}
Paprotny, A. and J.~Garcke (2012).
\newblock On a connection between maximum variance unfolding, shortest path
  problems and isomap.
\newblock In {\em Fifteenth International Conference on Artificial Intelligence
  and Statistics}.

\bibitem[\protect\citeauthoryear{Perrault-Joncas and Meila}{Perrault-Joncas and
  Meila}{2012}]{Perrault-Joncas}
Perrault-Joncas, D. and M.~Meila (2012).
\newblock Metric learning and manifolds: Preserving the intrinsic geometry.
\newblock Technical report, Department of Statistics, University of Washington.

\bibitem[\protect\citeauthoryear{Roweis and Saul}{Roweis and
  Saul}{2000}]{Roweis00LLE}
Roweis, S. and L.~Saul (2000).
\newblock Nonlinear dimensionality reduction by locally linear embedding.
\newblock {\em Science\/}~{\em 290\/}(5500), 2323--2326.

\bibitem[\protect\citeauthoryear{Saul, Weinberger, Ham, Sha, and Lee}{Saul
  et~al.}{2006}]{saul2006spectral}
Saul, L., K.~Weinberger, J.~Ham, F.~Sha, and D.~Lee (2006).
\newblock Semisupervised learning.
\newblock In B.~Schoelkopf, O.~Chapelle, and A.~Zien (Eds.), {\em Spectral
  methods for dimensionality reduction}, pp.\  293--308. MIT Press.

\bibitem[\protect\citeauthoryear{Singer}{Singer}{2006}]{singer2006graph}
Singer, A. (2006).
\newblock From graph to manifold laplacian: The convergence rate.
\newblock {\em Applied and Computational Harmonic Analysis\/}~{\em 21\/}(1),
  128--134.

\bibitem[\protect\citeauthoryear{Smith, Huo, and Zha}{Smith
  et~al.}{2008}]{smith2008convergence}
Smith, A., X.~Huo, and H.~Zha (2008).
\newblock Convergence and rate of convergence of a manifold-based dimension
  reduction algorithm.
\newblock In {\em Proceedings of Neural Information Processing Systems}, pp.\
  1529--1536. Citeseer.

\bibitem[\protect\citeauthoryear{Tenenbaum, de~Silva, and Langford}{Tenenbaum
  et~al.}{2000}]{Tenenbaum00ISOmap}
Tenenbaum, J.~B., V.~de~Silva, and J.~C. Langford (2000).
\newblock A global geometric framework for nonlinear dimensionality reduction.
\newblock {\em Science\/}~{\em 290\/}(5500), 2319--2323.

\bibitem[\protect\citeauthoryear{Van~der Maaten, Postma, and Van~den
  Herik}{Van~der Maaten et~al.}{2008}]{van2008dimensionality}
Van~der Maaten, L., E.~Postma, and H.~Van~den Herik (2008).
\newblock Dimensionality reduction: A comparative review.
\newblock Technical report, Tilburg University.

\bibitem[\protect\citeauthoryear{von Luxburg, Belkin, and Bousquet}{von Luxburg
  et~al.}{2008}]{vonLuxburg08}
von Luxburg, U., M.~Belkin, and O.~Bousquet (2008).
\newblock Consistency of spectral clustering.
\newblock {\em The Annals of Statistics\/}~{\em 36\/}(2), 555--586.

\bibitem[\protect\citeauthoryear{Weinberger, Packer, and Saul}{Weinberger
  et~al.}{2005}]{weinberger2005nonlinear}
Weinberger, K., B.~Packer, and L.~Saul (2005).
\newblock Nonlinear dimensionality reduction by semidefinite programming and
  kernel matrix factorization.
\newblock In {\em International Workshop on Artificial Intelligence and
  Statistics}, pp.\  381--388.

\bibitem[\protect\citeauthoryear{Weinberger, Sha, and Saul}{Weinberger
  et~al.}{2004}]{mvu04}
Weinberger, K., F.~Sha, and L.~Saul (2004).
\newblock Learning a kernel matrix for nonlinear dimensionality reduction.
\newblock In {\em International Confernence on Machine Learning (ICML)}, pp.\
  106.

\bibitem[\protect\citeauthoryear{Weinberger and Saul}{Weinberger and
  Saul}{2006}]{weinberger2006introduction}
Weinberger, K.~Q. and L.~K. Saul (2006).
\newblock An introduction to nonlinear dimensionality reduction by maximum
  variance unfolding.
\newblock In {\em National Conference on Artificial Intelligence (AAAI)},
  Volume~2, pp.\  1683--1686.

\bibitem[\protect\citeauthoryear{Weyl}{Weyl}{1939}]{MR1507388}
Weyl, H. (1939).
\newblock On the volume of tubes.
\newblock {\em Amer. J. Math.\/}~{\em 61\/}(2), 461--472.

\bibitem[\protect\citeauthoryear{Ye and Zhi}{Ye and
  Zhi}{2012}]{DiscreteHessian}
Ye, Q. and W.~Zhi (2012).
\newblock Discrete hessian eigenmaps method for dimensionality reduction.
\newblock Technical report, Department of Mathematics, University of Kentucky.

\bibitem[\protect\citeauthoryear{Zha and Zhang}{Zha and
  Zhang}{2003}]{zha2003isometric}
Zha, H. and Z.~Zhang (2003).
\newblock Isometric embedding and continuum isomap.
\newblock In {\em In Proceedings of the Twentieth International Conference on
  Machine Learning}. Citeseer.

\bibitem[\protect\citeauthoryear{Zhang and Zha}{Zhang and
  Zha}{2004}]{Zhang04LTSA}
Zhang, Z. and H.~Zha (2004).
\newblock Principal manifolds and nonlinear dimension reduction via tangent
  space alignment.
\newblock {\em SIAM J. Sci. Comput.\/}~{\em 26\/}(1), 313--338.

\end{thebibliography}

\end{document}